\documentclass{article} 
\usepackage{iclr2027_conference,times}

\usepackage{amsmath,amsfonts,bm}

\def\eqref#1{equation~\ref{#1}}

\def\1{\bm{1}}

\DeclareMathAlphabet{\mathsfit}{\encodingdefault}{\sfdefault}{m}{sl}
\SetMathAlphabet{\mathsfit}{bold}{\encodingdefault}{\sfdefault}{bx}{n}

\usepackage{hyperref}
\usepackage{url}
\usepackage{amsmath,amssymb,amsthm}
\usepackage{booktabs}
\usepackage{graphicx}
\usepackage{placeins}
\usepackage{needspace}
\usepackage{tikz}
\usetikzlibrary{decorations.pathreplacing,calc}
\usepackage{multirow}
\usepackage{xcolor}
\hypersetup{hidelinks}
\definecolor{figblue}{HTML}{306B9E}
\definecolor{figorange}{HTML}{BE673C}
\definecolor{figgray}{HTML}{AEB7BF}
\definecolor{figink}{HTML}{1A1A19}
\definecolor{figmut}{HTML}{6B6963}

\newcommand{\rhow}{\rho_w}
\newcommand{\rhob}{\rho_b}
\newcommand{\kbar}{\bar{k}}
\newcommand{\neff}{n_{\mathrm{eff}}}
\newcommand{\Hprov}{H_0^{\mathrm{prov}}}

\newtheorem{lemma}{Lemma}
\newtheorem{proposition}{Proposition}
\newtheorem{remark}{Remark}

\title{How Sensitive Are LLM Leaderboard Claims to Hidden Model Selection?}

\author{Chen Yang$^{1}$ \quad Jun Chen$^{2,*}$\\
$^{1}$Department of Statistics, Texas A\&M University\\
$^{2}$Division of Computational Biology,\\
Department of Quantitative Health Sciences, Mayo Clinic\\
$^{*}$Corresponding author: \texttt{Chen.Jun2@mayo.edu}}

\iclrfinalcopy
\begin{document}
\raggedbottom
\setlength{\parskip}{4pt}
\makeatletter
\setlength{\@fptop}{0pt}
\setlength{\@fpsep}{10pt}
\makeatother

\maketitle

\begin{abstract}
LLM leaderboard gains can reflect selection among privately evaluated
model variants, yet neither the number of variants nor their dependence
is public. We ask how many hidden variants a published margin can support
while retaining statistical evidence of a provider's advantage over a
fixed comparator. For a fixed candidate family under a Gaussian margin model, we derive a sensitivity
curve that reports this maximum count as a function of a lower bound on
within-family correlation. The relevant correlation must match the score
used for ranking and the sampling model: in a controlled family,
pooled item correlation is 0.90, whereas composite-score correlation is
0.46 under item resampling and 0.92 when MMLU subjects are resampled. An item-based audit
of 394 adjacent-rank claims on the Open LLM Leaderboard finds that 391
lack statistical support even before accounting for selection. Among
claims that pass the uncorrected test, certification can depend on
assumptions about the hidden family's correlation. The resulting curves
make these assumptions explicit without estimating the unobserved search size.
\end{abstract}

\section{Introduction}
\label{sec:intro}

Public leaderboards are the de facto arbiter of progress in large
language models, and their published claims influence deployment, funding,
and research agendas. Their margins rest on an unstated assumption:
that the submitted model was evaluated in isolation.
\citet{singh2025leaderboard} document otherwise: before one release,
Meta tested 27 Llama-4 variants on a battle-style platform and released
the best. We adopt 27 as a stated reference multiplicity (it was
documented on a different platform from the boards we audit). When
the submitted model is the maximum of $k$ privately scored siblings,
its margin is inflated, and marginal error bars
\citep{miller2024adding,bowyer2025position} are too short for the
selected claim. How much too short depends on how alike the siblings
are: $k$ private tries inflate a maximum in proportion to how
independently they vary, and sibling fine-tunes vary far from
independently.

We assess a provider's advantage over a fixed public comparator under
this selection model, not the correctness of the complete ranking.

A natural reaction is a multiplicity correction, but the correction is
underdetermined in two ways. The number of hidden variants $k$ is
unobserved, so no fixed correction applies. The hidden variants are
also not independent: sibling fine-tunes err on largely the same items
\citep{kim2025correlated}, so independence-based corrections are
miscalibrated, while dependence-robust Bonferroni is valid under any
dependence but conservative by an unmeasured amount. A public auditor
sees one submitted score and one comparator; the candidate set that
produced the submission, and its internal correlation, are private.

We therefore invert the question, in the spirit of
\citet{rosenthal1979file}'s file-drawer number: instead of correcting
for an unobserved $k$, we ask \emph{under how much hidden selection a
given claim remains certifiable}. Because the answer depends on the one property
no auditor observes, the hidden family's correlation, the output is a
curve rather than a number. We derive the least-favorable null
distribution of the selected margin under a Gaussian margin model and
invert it into a per-claim \emph{sensitivity curve} $\kbar(\rhow^-)$,
read as: if the hidden siblings correlate at least $\rhow^-$, the
published margin remains certified for at most $\kbar$ hidden tries. At
$\alpha = .05$ and comparator
correlation $\rhob = .5$, 27 correlated variants require the same
critical value as $5$--$19$ independent tries across our illustrative $\rhow$ scenarios; on a
numerical grid, the exact correction flips verdicts against the naive
test in $29.7\%$ of cells and against Bonferroni in $16.1\%$. The
effect on a given board depends on its observed margins.

Which correlation enters the curve turns out to be the most
consequential empirical question in the paper. Sibling fine-tunes
agree on most items, so a correlation pooled over all items is high:
$0.90$ in a controlled family of Qwen2.5 fine-tunes. But the board
ranks by an equal-weight mean of per-benchmark accuracies, and the
sampling variance under item resampling is dominated by its
\emph{smallest} benchmark. In the same family, GSM8K ($1{,}319$ items)
contributes $91\%$ of the composite's variance at an item-level sibling
correlation of $0.40$, while MMLU ($14{,}042$ items) contributes $9\%$ at
$0.94$. The pooled estimate is close to MMLU's correlation, whereas the
score-matched correlation of the
same family is $0.46$. Hidden-selection correction is determined by the correlation
on the score the board ranks by, so the correlation that matters is a
property of the model family, score functional, and sampling model
(Proposition~\ref{prop:functional}). Resampling MMLU subjects instead
of individual items raises this family's score correlation to $0.92$:
the direction of the pooled--score gap is not universal.

Under item resampling, the Open LLM Leaderboard audit finds that only
three of 394 adjacent-rank claims exceed the uncorrected threshold
(median $z=0.17$). Hidden selection therefore matters for a small subset
of otherwise resolvable gaps. The v1 rank-3 claim loses support at
$k\ge3$, whereas a v2 claim with $z=3.01$ survives all displayed
adjustments. For v1 rank 1 ($z=2.61$), accounting for selected nuisance
estimates raises the required correlation floor from $0.689$ to $0.849$.
This exceeds the minima measured under item resampling on our public
composites, but those measurements do not bound a private family.
The claim remains contingent on its unobservable correlation floor.
The package \texttt{selective-evals} computes these conditional
sensitivity curves from aligned per-item evaluation matrices.

\textbf{Contributions.}
\begin{itemize}
  \item \textbf{Certificate.} Exact Gaussian reference tail and conservative
  bounds for heterogeneous margins, inverted into per-claim sensitivity
  curves; estimated nuisances receive fixed-$K$ asymptotic adjustments
  (\S\ref{sec:certificate}).
  \item \textbf{The right correlation.} Pooled-item and score-matched
  correlations target different dependence structures
  (Proposition~\ref{prop:functional}); only the latter governs hidden
  selection on the audited score under the stated sampling model.
  \item \textbf{Measurement.} Score-matched, min-pairwise
  within-family correlations for sibling families, bracketed by a
  broad census and controlled experiments spanning three lineages and
  two update rules (\S\ref{sec:measurement}, Appendix~\ref{app:dial}).
  \item \textbf{Validation, audit, tool.} Masking finds calibration
  consistent with conservatism under item resampling; a conditional audit covers 394 v1
  claims, a v2 snapshot, and an Arena illustration;
  \texttt{selective-evals} includes the implementation and validation
  suite (\S\ref{sec:audit}; public release on acceptance).
\end{itemize}

\section{Hidden selection: setting and scope}
\label{sec:setup}

The audit fixes three objects: a provider privately scores $k$
siblings, submits the best, and the board displays that submitted score
next to a public comparator, with per-item correctness for both
displayed models. The observed margin is therefore a maximum over a
private family, not a single draw. The losing siblings, their number,
and their within-family correlation are private.

\textbf{Scores and correlation structure.} A static benchmark scores each model on $n$ shared items. Treating
items as a super-population
\citep{miller2024adding}, scores are asymptotically Gaussian and
correlate through item-level correctness. Three components suffice:
an item-difficulty component shared by every model, which cancels in
any margin; a family component shared by the siblings, which is what
makes them alike; and idiosyncratic noise, which is what selecting a
maximum tends to increase. Standardized to unit
variance, the working model for provider $p$ with $k$ hidden variants
and fixed comparator $q$ is the two-block structure
\begin{equation}
\label{eq:twoblock}
S_{pv} = \mu_{pv} + \sqrt{\rhob}\, Z_0 + \sqrt{\rhow-\rhob}\, F_p
+ \sqrt{1-\rhow}\,
\varepsilon_v \quad (v \le k), \qquad
S_q = \mu_q + \sqrt{\rhob}\, Z_0 + \sqrt{1-\rhob}\, \varepsilon_q,
\end{equation}
with $Z_0, F_p, \varepsilon_1,\dots,\varepsilon_k,\varepsilon_q$
independent standard Gaussians: variants correlate pairwise at
$\rhow$, variant and comparator at $\rhob$; $Z_0$ (shared item
difficulty) cancels in every margin.

The construction in \eqref{eq:twoblock} requires
$0\le\rhob\le\rhow\le1$. The certificate instead operates on the
\emph{margin vector} $(S_{pv}-S_q)_v$, with variance $2(1-\rhob)$
and pairwise covariance $1+\rhow-2\rhob$. Its nonnegative-correlation
reference law exists on $\rhow\ge2\rhob-1$, $\rhob<1$.
A joint, unit-variance score model additionally requires
$1+(k-1)\rhow-k\rhob^2\ge0$.
Outside score feasibility, the reference tail can still bound feasible
Gaussian margins by Slepian's inequality; it is not the exact law of a
hidden score family with those literal parameters
(Lemma~\ref{lem:robust}).

Even independent model scores produce correlated margins because each
subtracts the same comparator score. For example, $\rhow=\rhob=0$
gives margin variance $2$ and covariance $1$, hence correlation $1/2$.
The correction therefore depends on the margins' dependence, which
combines both score correlations.

\textbf{Estimand.} A provider submits the variant with the highest
private score. The observed margin on the leaderboard is
$M = \max_{v\le k} S_{pv} - S_q$, even though the losing variants
are never seen. The testable object is the
\emph{provider-level} null
\begin{equation}
\Hprov : \max_{v\le k} \mu_{pv} \le \mu_q,
\end{equation}
i.e., ``no variant of provider $p$ truly beats $q$.'' The audit
specifies the public pair, $k$, and a lower bound
on within-family correlation, then asks whether the published margin
still certifies the provider-level claim after $k$ hidden correlated
tries.
A rejection certifies that \emph{some} hidden variant truly beats $q$;
it need not be the submitted one. The same critical value also
controls, \emph{unconditionally}, the probability of certifying a
submitted model that is not truly better
(Proposition~\ref{prop:shipped}, Appendix~\ref{app:proofs}), but
conditional validity for the private winner would require unobserved
private margins.

\textbf{Conditional verdicts and applicability.} Every numerical verdict
is conditional on a fixed candidate family and a labeled pair $(k,\rhow^-)$.
\emph{Provider-certified}: margin exceeds the level-$\alpha$ critical
value at the assumed floor $\rhow^-$.
\emph{Not margin-certified}: the observed margin is insufficient to
reject $\Hprov$ after accounting for selection under the stated assumptions.
Failure to reject does not establish that no variant improves on the comparator.
\emph{Model-insufficient}: the domain check fails
($\rhow^- < 2\hat\rhob - 1$ for near-clone comparators) and only
Bonferroni applies.
\emph{Metadata-insufficient} is a separate applicability status when the
record does not establish a fixed candidate family. The public snapshots
provide no complete private candidate logs to this audit: applicability is
unverified for every row, even when a conditional numerical verdict is
reported. Selection by a private proxy is allowed within a fixed family
(Proposition~\ref{prop:shipped}); candidates generated adaptively from
benchmark feedback are outside the guarantee.

\section{From a margin to a hidden-selection sensitivity curve}
\label{sec:certificate}

\textbf{The least-favorable null.} The certificate must hold whatever
the hidden siblings' true means are, and the auditor never learns
them. The configuration most favorable to an inflated maximum is the
one in which every sibling is exactly as good as the comparator: any
sibling that is truly worse contributes less to the maximum, and any
that is truly better makes the null false.

\begin{lemma}[Least-favorability of equal means]
\label{lem:lfc}
Under $\Hprov$ and the Gaussian margin model, $\Pr(M > m)$ is maximized at
$\mu_{p1}=\cdots=\mu_{pk}=\mu_q$ for every $m$.
\end{lemma}

\textbf{Exact reference tail.} At the equal-means configuration the
exchangeable reference margin
has a one-factor form: with $G \sim N(0,\, 1+\rhow-2\rhob)$
independent of the noise,
\begin{equation}
\label{eq:tail}
\Pr\!\left(M > m\right) \;=\;
\mathbb{E}_G\!\left[\,1-\Phi\!\left(\frac{m-G}{\sqrt{1-\rhow}}\right)^{\!k}\,\right],
\end{equation}
Conditional on the shared term $G$, the $k$ margins have independent
residual noise. The $k$th power is therefore the probability that all
stay below $m$; subtracting from one gives the probability that at least
one exceeds it. Averaging over $G$ accounts for their shared variation
and requires only a one-dimensional integral.
This is the classical one-factor Dunnett
representation \citep{dunnett1955multiple}, mapped to the
hidden-selection margin. An auditor observes
$z = M / \sqrt{2(1-\rhob)}$. We write $p_k(z;\rhow,\rhob)$ for the
corresponding tail and $c_\alpha(k;\rhow,\rhob)$ for its
level-$\alpha$ critical value. At the clone boundary $\rhow \to 1$,
\eqref{eq:tail} reduces to the $k=1$ law: $k$ is unidentifiable
precisely where it is irrelevant.

\textbf{Certificate.} For an observed $z$, the certificate is the curve
\begin{equation}
\rhow^- \;\longmapsto\; \kbar(z;\rhow^-,\rhob) \;=\;
\sup\bigl(\{k \in \mathbb{N} : p_k(z;\rhow^-,\rhob) \le \alpha\}
\cup \{0\}\bigr) \in \mathbb{N}_0 \cup \{\infty\},
\end{equation}
the largest hidden multiplicity for which the claim remains certified
if the hidden family's correlation is at least $\rhow^-$ (proofs in
Appendix~\ref{app:proofs}).

For example, $\kbar(0.56)=17$ would certify the claim for up to 17
hidden variants, provided every sibling pair has correlation at least
$0.56$. It would not certify the claim at $k=27$ under that floor,
nor estimate how many variants the provider actually tested.

The curve is a hidden-selection sensitivity analysis. At the clone boundary
$\rhow^- \to 1$, the correction approaches the single-test threshold
because identical tries cannot inflate a maximum; at the independence boundary
$\rhow^- = 2\rhob - 1$, $\kbar$ recovers \v{S}id\'ak's finite count;
at $\rhow^- = \rhob$, the standard Dunnett limit. Between these
reference points, $k$ correlated hidden variants have the same critical value as
$\neff < k$ independent ones (Figure~\ref{fig:main}), and the curve
reports the unobserved within-family correlation assumption directly.

\textbf{Why a minimum, not a mean.} Real families are not
equicorrelated. Under the lemma's comparator and variance assumptions, replacing
all within-family correlations by their minimum gives an
equicorrelated upper bound on the selected tail. A family mean does
not supply that bound. For example, a family containing many nearly
identical variants and one weakly correlated variant can have a high
mean correlation. That additional variant still gives selection another
opportunity for an unusually high score; the mean can obscure it.
Figure~\ref{fig:validation-support}a shows the empirical
consequence of substituting a mean for a minimum.

\begin{lemma}[Monotonicity in $\rhow$, and the min-pairwise rule]
\label{lem:mono}
The least-favorable tail is nonincreasing in $\rhow$ on
$[2\rhob - 1, 1]$; and for any heterogeneous within-family correlation
matrix with $\min_{u \ne v} \rho_{uv} \ge 2\rhob - 1$ (homogeneous
comparator correlation), the equicorrelated tail at
$\rhow^- = \min_{u \ne v} \rho_{uv}$ is an upper bound (Slepian).
\end{lemma}

\textbf{Siblings need not match the comparator uniformly.} Each
sibling may correlate with the comparator differently and have a
different score variance. The same bound applies over a box of those
quantities, using a worst-case floor computed on the box.

\begin{lemma}[Robustness to comparator heterogeneity and bounded
variance ratios]
\label{lem:robust}
If within-family correlations satisfy $\rho_{uv} \ge \rhow^-$,
comparator correlations $\beta_v \le \beta^+$, and score SD ratios
$\sigma_v/\sigma_q \in [1/R, R]$, the one-factor tail
\eqref{eq:tail} at a worst-case floor $r^- = \inf r_{uv}$ (over the
standardized margins' correlation on this box) bounds the selected
margin under $\Hprov$ for any selection rule $W$. At $R = 1$,
$r^- = (1+\rhow^--2\beta^+)/(2-2\beta^+)$, independent of $\beta^-$;
for $R > 1$, $r^-$ is certified by interval branch and bound
(Appendix~\ref{app:robustproof}). If $r^- < 0$ the verdict is
\emph{model-insufficient}.
\end{lemma}

\noindent The bound is monotone in both caps ($r^-$ nonincreasing in
$\beta^+$ and, by box nesting, in $R$), so a Berger--Boos treatment of
$\beta_W$ needs only an upper endpoint
$\beta^+_{\mathrm{BB}} = U + \Delta$.

\begin{proposition}[Monotonicity in the comparator correlation]
\label{prop:rhobmono}
On $\rhow \ge 2\rhob - 1$, $\rhob < 1$, the least-favorable
standardized tail $p_k(z;\rhow,\rhob)$ is nondecreasing in $\rhob$: it
depends on $(\rhow,\rhob)$ only through
$r = (1+\rhow-2\rhob)/(2-2\rhob)$, nonincreasing in $\rhob$, and the
Gaussian maximum tail is nonincreasing in $r$ (strictly, in the
interior, for $k \ge 2$ and $\rhow < 1$).
\end{proposition}

\noindent Two reporting rules follow. The within-family floor is a
minimum, not a mean, labeled as a population assumption rather than a
universal constant; and floors outside the margin domain
$\rhow \ge 2\rhob - 1$ are reported as \emph{model-insufficient}, not
substituted into the tail. The same critical value gives an
unconditional submitted-model guarantee under true standardization
(Proposition~\ref{prop:shipped}, Appendix~\ref{app:proofs}).

\textbf{Accounting for estimated inputs.} The curve needs two more inputs the
auditor estimates from the winner's own public data, the comparator
correlation $\hat\rhob$ and the margin's standard error. These estimates
can also be affected by selection. We compare successive adjustments
that keep the same scope while accounting for one additional
selected quantity at each adjustment. Plug-in
treats the winner's denominator as fixed; selected-$\beta$ accounts
for the possibility that the winner was also selected through its
estimated comparator correlation; the joint adjustment covers selected
standardization from the raw margin. A public auditor sees $\beta_W$,
not the hidden vector
$(\beta_1,\dots,\beta_k)$, so the selected-$\beta$ adjustment uses
Berger--Boos. With Lemma~\ref{lem:robust}'s heterogeneity adjustment,
the selected-$\beta$ construction at count $K$ is
\begin{equation}
\label{eq:svbb}
p^{\mathrm{SVBB}}_K \;=\;
\min\bigl\{1,\; Q_K\bigl(z;\, r^-(R;\rhow^-,\beta^+_K)\bigr)
+ \gamma\bigr\},
\end{equation}
with $\beta^+_K = U_W(\gamma/K) + \Delta$ and
$\gamma = 0.005$ ($= \alpha/10$, following \citealp{berger1994p})
throughout. This column still uses a plug-in SE; the joint version
also bounds the SE using the raw margin (Appendix~\ref{app:svbb}).
Here $Q_K(z;r)$ is the upper tail of the maximum of $K$ standard
Gaussians with common correlation $r$. The fixed-candidate upper bound
$U_W(\eta)$ has nominal noncoverage $\eta$; using $\eta=\gamma/K$
accounts for selecting the winner. The assumed excess $\Delta$ bounds
how much another sibling's comparator correlation can exceed the winner's.
The added $\gamma$ allocates error probability to failure of the bound.
Exact nuisance bounds give finite-sample Gaussian validity only with
true standardization or a joint region covering it. Our influence-function
adjustments are pointwise-in-law asymptotic for fixed $K$.
For estimated nuisances we report the finite-range budget
$\widetilde{k}_{100}=\max(\{K\le100:p^{\mathrm{adj}}_K\le\alpha\}\cup\{0\})$,
over domain-valid counts, using the labeled adjustment; all-domain
failure is M-I. A value of $100$ is reported as $\ge100$, not an unlimited budget.
This reporting cap is not a finite-sample accuracy guarantee, nor a
simultaneous confidence statement over all correlation floors.
Appendix~\ref{app:svbb} gives the approximation conditions. The one nuisance no public retrospective audit
can estimate is $\rhow$ itself, because the hidden family is never
observed, which is why the output remains a curve.

\textbf{Effective multiplicity.} The independent-equivalent count
$\neff$ matches the Gaussian critical value to that of independent tries.
At $k=27$, $\alpha=.05$, and $\rhob=.5$, it ranges from $5$ to $19$
across the illustrative scenarios (Figure~\ref{fig:main}; numerical
comparisons and grid specification in Appendix~\ref{app:grid}).

\begin{figure}[!htbp]
\centering
\includegraphics[width=\linewidth]{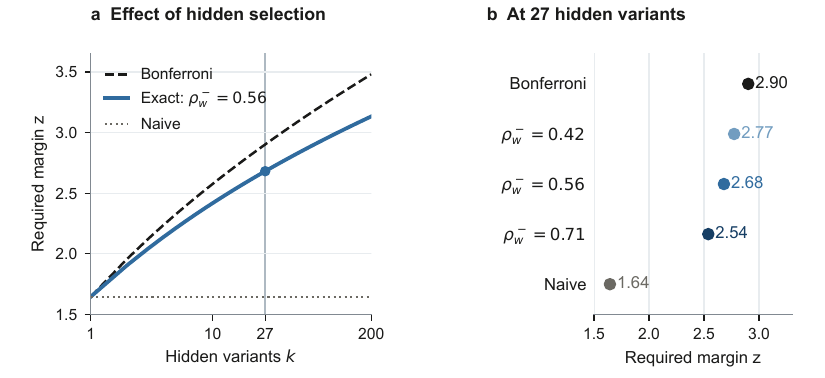}
\caption{\textbf{Hidden selection raises the margin needed for certification.}
A claim is certified when its observed standardized margin $z$ exceeds
the plotted threshold ($\rhob=0.5$, $\alpha=0.05$).
\textbf{(a)} Thresholds increase with hidden multiplicity; the exact curve
assumes $\rhow^-=0.56$. \textbf{(b)} At $k=27$, higher within-family
correlation lowers the exact threshold. The naive test ignores selection;
Bonferroni requires no correlation assumption. The count 27 is a sensitivity
reference, not an estimate of a provider's search.}
\label{fig:main}
\end{figure}
\vspace{0.5em}

\section{Correlation must match the ranked score}
\label{sec:measurement}

\textbf{Two correlation estimands.} The certificate is indexed by the correlation of
the score the board ranks by, and that is not the correlation most
readily measured. The Open LLM Leaderboard ranks by the equal-weight
mean of six per-benchmark accuracies. Under within-benchmark item
resampling, each benchmark's share of the composite's sampling
variance scales with $1/n_b$, so smaller benchmarks can contribute disproportionately;
a correlation pooled over items weights each benchmark by its item
share, so the largest benchmark weighs most. The two estimators
therefore report different benchmarks' correlations. The intro's
example (pooled $0.90$, score-matched $0.46$, in the Qwen2.5-1.5B
full-fine-tune stratum below) is this weighting at work, and in
$4/12$ observational families the two estimates disagree even on
which family is more correlated.

For intuition, consider two equally weighted benchmarks with the same
per-item variance but tenfold different item counts. The smaller benchmark
contributes ten times as much sampling variance to the composite, although
it supplies only one eleventh of the pooled items.

\begin{proposition}[Macro- and micro-weighted correlations target
different geometries]
\label{prop:functional}
Let the reported score be the equal-weight mean of $G$ per-benchmark
accuracies. Under within-benchmark item resampling, benchmark $b$
contributes $\operatorname{Var}_b/(G^2 n_b)$ to the score's sampling
variance, so its weight in the \emph{score} correlation scales as
$1/n_b$; a pooled-item covariance instead weights by item shares
$\pi_b$ and includes an extra between-benchmark term,
\[
\operatorname{Cov}_{\mathrm{pool}}(X_A, X_B) = \sum_b \pi_b\,
\sigma_{AB,b} + \sum_b \pi_b (\mu_{A,b}-\mu_A)(\mu_{B,b}-\mu_B),
\]
cross-benchmark difficulty the score correlation never sees. The two
estimands need not agree in magnitude \emph{or in ordering across
families}; the certificate requires the former.
\end{proposition}

\noindent Every sensitivity curve in \S\ref{sec:audit} therefore uses
the score-matched estimator.

\textbf{Variance decomposition.}
Figure~\ref{fig:score-summary} shows both the mechanism and its
controlled-family consequence. In five Qwen2.5-1.5B full-FT variants,
GSM8K is only $8.6\%$ of the items but contributes about $91\%$ of the
two-task score variance. Pooling emphasizes MMLU and adds a
between-benchmark mean term; independent task resampling emphasizes
GSM8K and has no such term. The resulting family medians are $0.894$
pooled and $0.459$ score-matched.

\begin{figure}[!htbp]
\centering
\includegraphics[width=\linewidth]{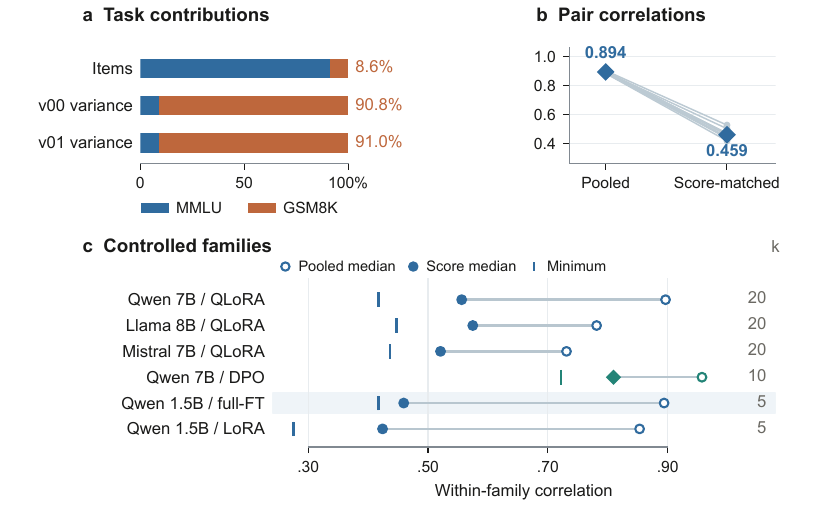}
\caption{\textbf{The ranked score changes the relevant correlation under item resampling.}
\textbf{(a)} In the Qwen2.5-1.5B full-FT worked family, GSM8K is a
small item share but dominates the two-task score variance
(blue: MMLU; orange: GSM8K).
\textbf{(b)} The same ten model pairs have high pooled-item
correlation and much lower ranked-score correlation; diamonds are
medians. \textbf{(c)} Across six controlled strata, pooled medians
remain high while score-matched medians and minima can be much lower.
Open points are pooled medians; filled points are score-matched
medians; ticks are score-matched minima. Subject-resampling sensitivity
is reported separately in Appendix~\ref{app:cluster}.}
\label{fig:score-summary}
\end{figure}

\textbf{Data and estimator.} The Open LLM Leaderboard per-item
correctness matrix \citep[via][]{polo2024tinybenchmarks}: 395 models
$\times$ 28{,}659 items across six benchmarks. We hand-curate sibling
families (same organization, same base model, an explicit variant
series in the name), yielding 12 families with $k = 3$--$8$
(deterministic rules and membership in
Appendix~\ref{app:measurement}). Under within-benchmark item
resampling, $\operatorname{Cov}(s_A, s_B) = \tfrac{1}{36}\sum_b
\operatorname{Cov}_b(X_A, X_B)/n_b$; this score-matched estimate is
primary, with item-bootstrap confidence intervals stratified by
family and benchmark. Pooling instead shifts the median from $0.707$
to $0.748$ and reverses the gap's sign for $4/12$ families.

\textbf{Results (Appendix Figure~\ref{fig:family-evidence}a).} Score-matched family
means span $0.658$--$0.855$ (median $0.707$); the min-of-mins across
the 12 families is $0.565$ under item resampling. Slepian requires a population minimum
for the hidden family; a public sample minimum does not supply it.
Item-bootstrap CIs are about $\pm0.005$ under independent item sampling,
excluding subject dependence and benchmark-construction uncertainty.
We therefore use these measurements to illustrate scenarios, not to
calibrate a likely hidden-family floor.
Near-clone pairs reach $0.985$. Between providers, $\hat\rhob = 0.500$
(accuracy-matched $0.601$). Per benchmark, median within-family
$\hat\rhow$ ranges from $0.44$ (GSM8K) to $0.83$
(TruthfulQA/HellaSwag).

\textbf{Beyond hand-picked families.} A census of v2
submitter-declared \texttt{base\_model} tags (45 org-disjoint
families, per-item MMLU-Pro) broadens the measurements: median family mean
$0.746$, quartiles $[0.59, 0.80]$, and 17/45 below $0.56$
min-pairwise. \citet{kim2025correlated}'s independent dataset gives
the same ordering (same-org version series $0.778$, cross-base
$0.475$, matching our $\hat\rhob$), and prompt-format variation on
\citet{promteval2024}'s $15 \times 100$ template grid moves same-base
$\hat\rhow$ only modestly (IQR $0.019$--$0.033$) against
cross-provider $0.33$--$0.50$ (Appendix~\ref{app:measurement}). The
three controlled families whose $k = 20$ approaches the audit's
$k = 27$ have score-matched minima $0.42$--$0.45$, at the census
lower tail. We use $0.42/0.56/0.71$ as \emph{illustrative sensitivity anchors},
not estimated or conservative lower bounds for private families. Observed
minima also depend on family size and truncation (Appendix~\ref{app:measurement}).

\textbf{Controlled strata.} Because observational families are found
rather than made, we train six controlled strata to reduce reliance on
curated observational families (Figure~\ref{fig:score-summary}c; Appendix~\ref{app:dial}): 20 QLoRA
variants each of Qwen2.5-7B, Llama-3.1-8B and Mistral-7B-v0.3
(pre-specified factors: data mixture, rank, steps, seed), 10 DPO
variants of the same Qwen base, and two $5$-variant strata of
Qwen2.5-1.5B (full fine-tune and LoRA), all evaluated per-item on
MMLU$+$GSM8K. At the pooled item level every SFT stratum is uniformly
correlated (medians $0.73$--$0.90$), but score-matched medians on the
audited composite drop to $0.42$--$0.58$: the selection-relevant
correlation is not a family invariant. Three further patterns matter
for the audit. The DPO family is the informative outlier: preference
optimization moves the model so little that its score-matched median
stays at $0.81$ (pooled $0.96$), illustrating why one update rule
cannot calibrate the correlation of another. The Qwen2.5-1.5B full-FT stratum reaches the
edge of the equal-variance model ($R = 1.17$ against an external
comparator), and Lemma~\ref{lem:robust}'s joint $(\beta^+, R)$ column
restores coverage; Table~\ref{tab:audit} therefore defaults to
$R = 1$ and reports the $R > 1$ stress test separately. Across all
six strata, the certificate-relevant excess
$\max_v \beta_v - \beta_W$ has median $0.010$ and maximum $0.121$,
absorbed into \S\ref{sec:audit}'s adjusted verdicts. Selection
replays ($B = 10^5$) over $18$ family-stratum$\times$comparator
combinations find no resolved homogeneous-tail failure in $17$ of $18$;
the single resolved failure is the equal-variance cell whose
empirical pointwise domination the $R$ column restores
(Appendix~\ref{app:dial}).
\par

\textbf{Changing the sampling unit.} Resampling MMLU's 57 subjects
as whole blocks, while keeping its item-weighted score and resampling
other benchmarks by item, changes the worked family's score-matched
median from $0.459$ to $0.921$ ($B=20{,}000$; Appendix~\ref{app:cluster}).
Thus the pooled--score contrast is conditional on the sampling model.
For v1 \#1, cluster-matched standardization changes $z$ from $2.61$ to
$2.53$ and comparator correlation from $0.567$ to $0.759$; the plug-in
required floor at $k=27$ rises from $0.689$ to $0.861$.
At $0.56$ the claim remains not margin-certified; at $0.42$ it becomes
model-insufficient. The \#3 margin no longer passes the Gaussian naive
threshold. These checks do not address unknown template or contamination
dependence within GSM8K.

\textbf{Alternative explanations.}
Under item resampling, the pooled--score difference persists under strict parsing, restricting to the
best-scoring variants, and matching pairs by accuracy
(Appendix Figure~\ref{fig:controls}). These controls use the combined ten-variant
full-FT/LoRA pool: its $0.294$ median is not the $0.459$ median of the
five full-FT variants in Figure~\ref{fig:score-summary}b.
Filtering changes the candidate pool and its observed minimum; it
cannot establish a universal hidden-family floor. Detailed control
protocols and replays are in Appendix~\ref{app:dial}.

\section{Auditing public leaderboards}
\label{sec:audit}

\textbf{Protocol.} Every verdict below is \emph{pair-conditional}:
the audit takes the published pair as fixed and corrects only the
provider's unobserved within-family selection, not which pair became a
published claim (an observable layer; Appendix~\ref{app:masking} calibrates
it). Under independent item resampling, we audit cross-provider claims using official-style
$z$ with plug-in paired SE and score-matched $\hat\rhob$ per pair.
Verdicts are stated at $\rhow^- \in \{0.42, 0.56\}$ (median $0.71$
as an illustrative sensitivity anchor), with selected-$\beta$ and joint
Berger--Boos adjustments for claims with $z > z_{1-\alpha}$ (domain check on the
\emph{cap}). $k = 27$ is the one \emph{documented} multiplicity (on
Arena, not these boards), stated as an assumption alongside each
claim's budget and the winner's public submission count. These are
conditional scenarios; submission counts do not establish private search
size or nonadaptive provenance. Because the
realized comparator is itself a maximum over the board, it is
deterministically protective for the raw margin and calibrated by
replay for the studentized one. We validate the model tail by masking
curated families and replaying selection under real dependence
($B = 10^5$): it is conservative in 12/12 families, and the
min-pairwise column dominates on the audited composite in 12/12
(Appendix~\ref{app:masking}). The one resolved controlled-family
exception is the equal-variance full-FT/Yi cell; the joint
$(\beta^+,R)$ column restores domination there
(Appendix Figure~\ref{fig:validation}).

\textbf{Most adjacent margins are unresolved before selection.}
Across all 394 v1 adjacent cross-provider claims (overlapping, not
independent), the median is $z = 0.17$ and the 95th percentile is
$0.78$ \citep[cf.][]{kotawala2026resolution}. Even the naive
(uncorrected) column certifies only 3 of 394; the Gaussian plug-in column
certifies 0 (33 model-insufficient, 361 refused), Bonferroni 1.
The Gaussian column withdraws support from these three: two become
not margin-certified and one becomes model-insufficient. Bonferroni
still supports the latter. Aggregate decisions are stable,
and the certificate is a per-claim diagnostic for the few claims that
have a nonzero certification question; Appendix Figure~\ref{fig:audit}
shows the $(z,\hat\rhob)$ geometry and verdict counts. \par

\begin{table}[!htbp]
\caption{Conditional top-claim audit at $k=27$, $\alpha=0.05$, $\rhow^-=0.56$.
$\checkmark$: certified under the stated assumptions; $\times$: not margin-certified;
M-I: model-insufficient. Gaussian columns use $R=1$ and independent item resampling; adjusted
columns use asymptotic nuisance bounds. Fixed-family provenance is unverified for all rows.
M-I concerns the margin-correlation domain, not evidence against an advantage.}
\label{tab:audit}
\centering
\small
\begin{tabular}{llccccccc}
\toprule
& & & & \multicolumn{2}{c}{Baselines} &
\multicolumn{3}{c}{Gaussian model$^{R=1}$(.56), by adjustment} \\
\cmidrule(lr){5-6}\cmidrule(lr){7-9}
Board & Rank & $z$ & $\hat\rhob$ & Naive & Bonf & plug-in &
\shortstack{adjusted\\correlation} & \shortstack{adjusted corr.\\and SE} \\
\midrule
v1 & \#1 & 2.61 & .567 & \checkmark & $\times$ & $\times$ & $\times$ & $\times$ \\
v1 & \#2 & 3.26 & .874 & \checkmark & \checkmark & M-I & M-I & M-I \\
v1 & \#3 & 1.81 & .504 & \checkmark & $\times$ & $\times$ & $\times$ & $\times$ \\
v2 & \#1 & 2.44 & .750 & \checkmark & $\times$ & $\times$ & M-I & M-I \\
v2 & \#4 & 3.01 & .434 & \checkmark & \checkmark & \checkmark & \checkmark & \checkmark \\
\bottomrule
\end{tabular}
\end{table}
\vspace{0.4em}

\textbf{Top claims (Table~\ref{tab:audit}).}
The five rows separate three regimes. v1 \#1 is the only
$k$-sensitive verdict: plug-in budgets $\kbar = 13$--17 exceed Bonferroni's 11
at the measured anchors, but the claim is refused from $k=20$ at the
$0.56$ floor (Figure~\ref{fig:audit-cases}; Appendix~\ref{app:audit-extra}).
v1 \#2 is Bonferroni-certified; the Gaussian model abstains because
$0.56<2(0.874)-1=0.748$ violates its nonnegative-margin-correlation domain;
v1 \#3 is an uncorrected-significance case, refused by every correction
at $k \ge 3$. On v2, \#1 is refused or M-I at every adjustment, whereas
the top-five \#4 claim certifies throughout; two identical resubmissions
clear every column at $z=11$, and the top-30 extension certifies $5/30$
naively and $2$ at the curated floor (Appendix~\ref{app:audit-extra}).

\Needspace{4\baselineskip}
\textbf{The rank-1 claim: a verdict that depends on one
unobservable.} Each adjustment of Table~\ref{tab:audit} is a successive
sensitivity check that accounts for one additional selected quantity.
Figure~\ref{fig:sensitivity} separates these nuisance treatments
from additional hidden-family assumptions. The required floor rises
from $0.689$ under plug-in to $0.849$ under joint adjustment,
so once both within-pair selection channels are accounted for,
certifying the \#1 claim at $k = 27$ requires a
\emph{minimum-pairwise} floor of $0.849$ (Figure~\ref{fig:sensitivity}a).
Under item resampling, the largest minima on the curated and controlled composites are
$0.792$ and $0.72$, respectively. Four census families exceed $0.849$
on MMLU-Pro alone; those measurements do not establish a floor for
the audited composite. At the primary floors $0.42$ and $0.56$, no
v1 claim is Gaussian-certified under any adjustment. Selected-$\beta$
and joint adjustments increase model-insufficient counts from $69$ to
$84$ and from $33$ to $40$, respectively. At $0.707$, the one plug-in
certification is lost after nuisance adjustment; at $0.86$, both
certifications remain. These comparisons hold $(\Delta,R)=(0,1)$;
the following stress test adds heterogeneity and variance assumptions.

\begin{figure}[!htbp]
\centering
\includegraphics[width=\linewidth]{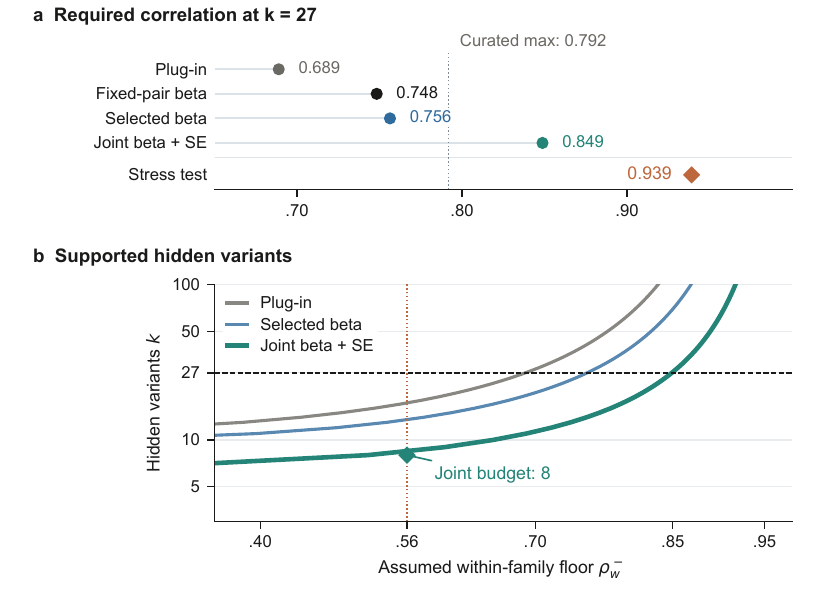}
\caption{\textbf{Assumption sensitivity for v1 claim \#1.}
\textbf{(a)} Required floors at $k=27$; dotted line: largest curated
minimum ($0.792$). The separate stress row adds
$(\Delta,R)=(0.121,1.16)$. \textbf{(b)} Boundaries with nuisance caps
recomputed at every $k$, holding $(\Delta,R)=(0,1)$.
Conditional support lies to the right; adjusted curves are asymptotic
and reported only for $K\le100$. All curves use independent item resampling. Guides mark $k=27$ and
$\rhow^-=0.56$; the marked point gives the joint budget of eight at
that floor. Vertical readings give finite-range budgets; reaching
the upper boundary means $\ge100$, with no extrapolation beyond it.}
\label{fig:sensitivity}
\end{figure}

\textbf{Measured-extreme stress test.} Adding the measured $\beta^+$
excess ($\Delta = 0.121$) tightens the plug-in rank-1 budgets at
$\rhow^- = 0.42/0.56/0.71/0.86$ to $11/13/18/66$ with no
certification flips; adding the measured variance ratio ($R = 1.16$)
moves the $0.42$ anchor outside the model domain and leaves budgets
$12/16/43$ at $0.56/0.71/0.86$. Combining these adjustments pushes the
rank-1 required floor to $0.939$, so every measured sensitivity anchor
refuses ($p = 0.071$ even at $0.86$). The test is calibrated to
observed extremes rather than to any named hidden family, and it
mainly increases abstention (Appendix~\ref{app:audit-extra}).

\section{Discussion and limitations}
\label{sec:discussion}

\textbf{What the certificate reports.}
The Gaussian reference tail is exact with known nuisances; heterogeneous
families receive a conservative bound. Estimated-input audits are plug-in
diagnostics or fixed-$K$ asymptotic adjustments. On dense leaderboards,
ordinary sampling uncertainty already limits most comparisons. For the
remaining margins, a sensitivity curve makes the hidden-family assumptions
explicit without claiming to estimate them.

\textbf{Assumptions and scope.}
Lemma~\ref{lem:robust} handles bounded variance ratios and comparator
heterogeneity, not non-Gaussian tails or unbounded variance asymmetry.
The item-based audit and the subject-cluster sensitivity analysis target
different sampling populations; neither accounts for all benchmark
construction or contamination uncertainty. Hidden-family floors and
heterogeneity caps remain assumptions, not estimates for frontier labs.
The audit is pair-conditional and fixed-$k$; overlapping board counts
are descriptive. Arbitrary selection within a fixed family is allowed,
but generating new candidates through benchmark feedback is outside the
guarantee \citep[cf.][]{roelofs2019meta}. Since private candidate logs
are unavailable, all public-board verdicts remain conditional.

\textbf{What a board should publish.}
\looseness=-1
A leaderboard operator can request a candidate count and a record of
how the candidate set was fixed before benchmark feedback, then publish
a finite-range sensitivity curve with its assumptions and applicability
status (cf.~\citet{singh2025leaderboard}). A declared count alone does not
verify that protocol. With all candidates disclosed, joint max-$t$
inference can reduce the nuisance penalty; independent confirmation data
can instead separate selection from evaluation. These prospective options
complement retrospective sensitivity analysis
(Appendices~\ref{app:svbb} and~\ref{app:toolexample}).

\bibliography{refs}

\begin{thebibliography}{26}
\providecommand{\natexlab}[1]{#1}
\providecommand{\url}[1]{\texttt{#1}}
\expandafter\ifx\csname urlstyle\endcsname\relax
  \providecommand{\doi}[1]{doi: #1}\else
  \providecommand{\doi}{doi: \begingroup \urlstyle{rm}\Url}\fi

\bibitem[Aronow et~al.(2015)Aronow, Samii, and Assenova]{aronow2015cluster}
Peter~M. Aronow, Cyrus Samii, and Valentina~A. Assenova.
\newblock Cluster-robust variance estimation for dyadic data.
\newblock \emph{Political Analysis}, 23\penalty0 (4):\penalty0 564--577, 2015.
\newblock \doi{10.1093/pan/mpv018}.

\bibitem[Bakshi et~al.(2026)Bakshi, Gao, Gao, and
  Panigrahi]{bakshi2026flexible}
Soham Bakshi, Lingjun Gao, Zijun Gao, and Snigdha Panigrahi.
\newblock Flexible inference for winners with conditional validity.
\newblock \emph{arXiv preprint arXiv:2607.18545}, 2026.
\newblock \doi{10.48550/arXiv.2607.18545}.

\bibitem[Berger \& Boos(1994)Berger and Boos]{berger1994p}
Roger~L. Berger and Dennis~D. Boos.
\newblock P values maximized over a confidence set for the nuisance parameter.
\newblock \emph{Journal of the American Statistical Association}, 89\penalty0
  (427):\penalty0 1012--1016, 1994.
\newblock \doi{10.1080/01621459.1994.10476836}.

\bibitem[Blum \& Hardt(2015)Blum and Hardt]{blum2015ladder}
Avrim Blum and Moritz Hardt.
\newblock The {L}adder: A reliable leaderboard for machine learning
  competitions.
\newblock In \emph{Proceedings of the 32nd International Conference on Machine
  Learning}, volume~37 of \emph{Proceedings of Machine Learning Research}, pp.\
   1006--1014. PMLR, 2015.
\newblock arXiv:1502.04585.

\bibitem[Bowyer et~al.(2025)Bowyer, Aitchison, and Ivanova]{bowyer2025position}
Sam Bowyer, Laurence Aitchison, and Desi~R. Ivanova.
\newblock Position: Don't use the {CLT} in {LLM} evals with fewer than a few
  hundred datapoints.
\newblock In \emph{Proceedings of the 42nd International Conference on Machine
  Learning}, volume 267 of \emph{Proceedings of Machine Learning Research},
  pp.\  81143--81184. PMLR, 2025.
\newblock URL \url{https://proceedings.mlr.press/v267/bowyer25a.html}.
\newblock arXiv:2503.01747.

\bibitem[Dunnett(1955)]{dunnett1955multiple}
Charles~W. Dunnett.
\newblock A multiple comparison procedure for comparing several treatments with
  a control.
\newblock \emph{Journal of the American Statistical Association}, 50\penalty0
  (272):\penalty0 1096--1121, 1955.
\newblock \doi{10.1080/01621459.1955.10501294}.

\bibitem[Dwork et~al.(2015)Dwork, Feldman, Hardt, Pitassi, Reingold, and
  Roth]{dwork2015reusable}
Cynthia Dwork, Vitaly Feldman, Moritz Hardt, Toniann Pitassi, Omer Reingold,
  and Aaron Roth.
\newblock The reusable holdout: Preserving validity in adaptive data analysis.
\newblock \emph{Science}, 349\penalty0 (6248):\penalty0 636--638, 2015.
\newblock \doi{10.1126/science.aaa9375}.

\bibitem[Hardt(2017)]{hardt2017climbing}
Moritz Hardt.
\newblock Climbing a shaky ladder: Better adaptive risk estimation.
\newblock \emph{arXiv preprint arXiv:1706.02733}, 2017.
\newblock \doi{10.48550/arXiv.1706.02733}.
\newblock URL \url{https://arxiv.org/abs/1706.02733}.

\bibitem[Hung \& Fithian(2019)Hung and Fithian]{hung2019rank}
Kenneth Hung and William Fithian.
\newblock Rank verification for exponential families.
\newblock \emph{The Annals of Statistics}, 47\penalty0 (2):\penalty0 758--782,
  2019.
\newblock \doi{10.1214/17-AOS1634}.

\bibitem[Kim et~al.(2025)Kim, Garg, Peng, and Garg]{kim2025correlated}
Elliot~Myunghoon Kim, Avi Garg, Kenny Peng, and Nikhil Garg.
\newblock Correlated errors in large language models.
\newblock In \emph{Proceedings of the 42nd International Conference on Machine
  Learning}, volume 267 of \emph{Proceedings of Machine Learning Research},
  pp.\  30038--30066. PMLR, 2025.
\newblock URL \url{https://proceedings.mlr.press/v267/kim25e.html}.
\newblock arXiv:2506.07962.

\bibitem[Kotawala(2026)]{kotawala2026resolution}
Anany Kotawala.
\newblock Resolution diagnostics for paired {LLM} evaluation.
\newblock \emph{arXiv preprint arXiv:2605.30315}, 2026.
\newblock \doi{10.48550/arXiv.2605.30315}.
\newblock ICML 2026 Workshop on Hypothesis Testing.

\bibitem[Lee et~al.(2016)Lee, Sun, Sun, and Taylor]{lee2016exact}
Jason~D. Lee, Dennis~L. Sun, Yuekai Sun, and Jonathan~E. Taylor.
\newblock Exact post-selection inference, with application to the lasso.
\newblock \emph{The Annals of Statistics}, 44\penalty0 (3):\penalty0 907--927,
  2016.
\newblock \doi{10.1214/15-AOS1371}.

\bibitem[McCloskey \& Michaillat(2026)McCloskey and
  Michaillat]{mccloskey2024critical}
Adam McCloskey and Pascal Michaillat.
\newblock Critical values robust to p-hacking.
\newblock \emph{Review of Economics and Statistics}, 108\penalty0 (4):\penalty0
  1085--1097, 2026.
\newblock \doi{10.1162/rest_a_01456}.
\newblock arXiv:2005.04141.

\bibitem[Messing(2026)]{messing2026hidden}
Solomon Messing.
\newblock Hidden measurement error in {LLM} pipelines distorts annotation,
  evaluation, and benchmarking.
\newblock \emph{arXiv preprint arXiv:2604.11581}, 2026.
\newblock \doi{10.48550/arXiv.2604.11581}.
\newblock URL \url{https://arxiv.org/abs/2604.11581}.

\bibitem[Miller(2024)]{miller2024adding}
Evan Miller.
\newblock Adding error bars to evals: A statistical approach to language model
  evaluations.
\newblock \emph{arXiv preprint arXiv:2411.00640}, 2024.
\newblock \doi{10.48550/arXiv.2411.00640}.

\bibitem[Min et~al.(2025)Min, Pang, Du, Liu, Cheng, and Lin]{min2025rigging}
Rui Min, Tianyu Pang, Chao Du, Qian Liu, Minhao Cheng, and Min Lin.
\newblock Improving your model ranking on chatbot arena by vote rigging.
\newblock In \emph{Proceedings of the 42nd International Conference on Machine
  Learning}, volume 267 of \emph{Proceedings of Machine Learning Research},
  pp.\  44252--44271. PMLR, 2025.
\newblock URL \url{https://proceedings.mlr.press/v267/min25a.html}.
\newblock arXiv:2501.17858.

\bibitem[Neufeld et~al.(2026)Neufeld, Perry, and Witten]{neufeld2026selective}
Anna Neufeld, Ronan Perry, and Daniela Witten.
\newblock Inference conditional on selection: A review.
\newblock \emph{arXiv preprint arXiv:2604.09779}, 2026.
\newblock \doi{10.48550/arXiv.2604.09779}.
\newblock URL \url{https://arxiv.org/abs/2604.09779}.

\bibitem[Polo et~al.(2024{\natexlab{a}})Polo, Weber, Choshen, Sun, Xu, and
  Yurochkin]{polo2024tinybenchmarks}
Felipe~Maia Polo, Lucas Weber, Leshem Choshen, Yuekai Sun, Gongjun Xu, and
  Mikhail Yurochkin.
\newblock tiny{B}enchmarks: Evaluating {LLM}s with fewer examples.
\newblock In \emph{Proceedings of the 41st International Conference on Machine
  Learning}, volume 235 of \emph{Proceedings of Machine Learning Research},
  pp.\  34303--34326. PMLR, 2024{\natexlab{a}}.
\newblock URL \url{https://proceedings.mlr.press/v235/maia-polo24a.html}.

\bibitem[Polo et~al.(2024{\natexlab{b}})Polo, Xu, Weber, Silva, Bhardwaj,
  Choshen, de~Oliveira, Sun, and Yurochkin]{promteval2024}
Felipe~Maia Polo, Ronald Xu, Lucas Weber, M{\'i}rian Silva, Onkar Bhardwaj,
  Leshem Choshen, Allysson Flavio~Melo de~Oliveira, Yuekai Sun, and Mikhail
  Yurochkin.
\newblock Efficient multi-prompt evaluation of {LLM}s.
\newblock In \emph{Advances in Neural Information Processing Systems
  (NeurIPS)}, volume~37, pp.\  22483--22512. Curran Associates, Inc.,
  2024{\natexlab{b}}.
\newblock \doi{10.52202/079017-0707}.
\newblock arXiv:2405.17202.

\bibitem[Roelofs et~al.(2019)Roelofs, Shankar, Recht, Fridovich-Keil, Hardt,
  Miller, and Schmidt]{roelofs2019meta}
Rebecca Roelofs, Vaishaal Shankar, Benjamin Recht, Sara Fridovich-Keil, Moritz
  Hardt, John Miller, and Ludwig Schmidt.
\newblock A meta-analysis of overfitting in machine learning.
\newblock In \emph{Advances in Neural Information Processing Systems
  (NeurIPS)}, volume~32, pp.\  9179--9189, 2019.
\newblock URL
  \url{https://proceedings.neurips.cc/paper_files/paper/2019/hash/ee39e503b6bedf0c98c388b7e8589aca-Abstract.html}.

\bibitem[Rosenthal(1979)]{rosenthal1979file}
Robert Rosenthal.
\newblock The ``file drawer problem'' and tolerance for null results.
\newblock \emph{Psychological Bulletin}, 86\penalty0 (3):\penalty0 638--641,
  1979.
\newblock \doi{10.1037/0033-2909.86.3.638}.

\bibitem[Singh et~al.(2025)Singh, Nan, Wang, Dsouza, Kapoor, {\"U}st{\"u}n,
  Koyejo, Deng, Longpre, Smith, Ermis, Fadaee, and
  Hooker]{singh2025leaderboard}
Shivalika Singh, Yiyang Nan, Alex Wang, Daniel Dsouza, Sayash Kapoor, Ahmet
  {\"U}st{\"u}n, Sanmi Koyejo, Yuntian Deng, Shayne Longpre, Noah~A. Smith,
  Beyza Ermis, Marzieh Fadaee, and Sara Hooker.
\newblock The leaderboard illusion.
\newblock In \emph{Advances in Neural Information Processing Systems},
  volume~38, pp.\  86910--86964. Curran Associates, Inc., 2025.
\newblock \doi{10.52202/085713-2620}.

\bibitem[Sood(2025)]{sood2025powerful}
Anav Sood.
\newblock Powerful rank verification for multivariate {G}aussian data with any
  covariance structure.
\newblock \emph{arXiv preprint arXiv:2503.01065}, 2025.
\newblock \doi{10.48550/arXiv.2503.01065}.

\bibitem[Xu et~al.(2026)Xu, Zhang, Sun, Zhou, Cao, and Aggarwal]{xu2026siren}
Yang Xu, Jiefu Zhang, Haixiang Sun, Zihan Zhou, Tianyu Cao, and Vaneet
  Aggarwal.
\newblock Towards reliable {LLM} evaluation: Correcting the winner's curse in
  adaptive benchmarking.
\newblock \emph{arXiv preprint arXiv:2605.05973}, 2026.
\newblock \doi{10.48550/arXiv.2605.05973}.
\newblock URL \url{https://arxiv.org/abs/2605.05973}.

\bibitem[Zrnic \& Fithian(2024{\natexlab{a}})Zrnic and
  Fithian]{zrnic2024locally}
Tijana Zrnic and William Fithian.
\newblock Locally simultaneous inference.
\newblock \emph{The Annals of Statistics}, 52\penalty0 (3):\penalty0
  1227--1253, 2024{\natexlab{a}}.
\newblock \doi{10.1214/24-AOS2391}.

\bibitem[Zrnic \& Fithian(2024{\natexlab{b}})Zrnic and Fithian]{zrnic2024zoom}
Tijana Zrnic and William Fithian.
\newblock A flexible defense against the winner's curse.
\newblock \emph{arXiv preprint arXiv:2411.18569}, 2024{\natexlab{b}}.
\newblock \doi{10.48550/arXiv.2411.18569}.

\end{thebibliography}
\bibliographystyle{iclr2027_conference}

\appendix

\section{Ethics, reproducibility, and AI use}
\label{app:statements}

\textbf{Ethics.}
The audit evaluates \emph{public claims} (published rankings and margins),
not private conduct. A \emph{not margin-certified} verdict states that
a published margin is statistically compatible with the provider-level
null under a stated selection model at a stated $(k,\rhow)$. It makes
no claim about whether any provider performed hidden testing. No
verdict is labeled ``refuted,'' by construction. All data are public
releases (per-item leaderboard matrices, submission metadata, battle
records); no user data or private model outputs are accessed.

\textbf{Reproducibility.}
The supplementary material provides the \texttt{selective-evals}
package, analysis and figure-generation scripts, and machine-readable
results. The package includes the exact quadrature and certificate,
with tests for Monte Carlo agreement, boundary laws, monotonicity,
inversion consistency, and the rank-1 sensitivity curve. Figure generation
uses the included numerical results; rerunning the experiments requires
the corresponding public datasets or model-evaluation outputs.
Appendix~\ref{app:grid}--\ref{app:masking} describe the experimental protocols.

\textbf{AI use.}
Large language models were used substantially in preparing this work:
for code development and experiment orchestration, literature search,
adversarial review simulation, proof drafting and checking, and
language editing. All mathematical arguments, empirical analyses, and
reported numbers were independently examined and verified by the
authors, who take full responsibility for the content of the paper.

\section{Proofs}
\label{app:proofs}

Throughout, write $\delta_v = \mu_{pv} - \mu_q$ and
$s = \sqrt{1-\rhow}$. All proofs run at the \emph{margin level}: on the
full domain $\rhow \in [2\rhob - 1,\, 1]$, the margin vector is the
exchangeable Gaussian
\begin{equation}
\label{eq:margin-decomp}
D_v \;=\; \delta_v \;+\; s\,\varepsilon_v \;+\; G, \qquad
G \sim N(0,\; 1+\rhow-2\rhob) \text{ independent of }
(\varepsilon_1,\dots,\varepsilon_k),
\end{equation}
which has the moments of \S\ref{sec:setup}
($\operatorname{Var} = 2-2\rhob$, $\operatorname{Cov} = 1+\rhow-2\rhob$),
with $M = \max_{v \le k} D_v$. For $\rhow \ge \rhob$ this is exactly
the margin law of the score representation \eqref{eq:twoblock}: take
$G = \sqrt{\rhow-\rhob}\, F_p - \sqrt{1-\rhob}\, \varepsilon_q$ and
note the common factor $Z_0$ cancels in every margin. For
$2\rhob - 1 \le \rhow < \rhob$ the score-level factorization does not
exist, but \eqref{eq:margin-decomp} remains the unique exchangeable
Gaussian with those moments --- the object the certificate actually
tests, per \S\ref{sec:setup}. $\Hprov$ is the constraint
$\delta_v \le 0$ for all $v$.

\subsection[Proof of Lemma 1]{Proof of Lemma~\ref{lem:lfc} (least-favorability of equal means)}

Condition on $G = g$. For $\rhow < 1$ ($s>0$), by independence of the
$\varepsilon_v$,
\[
\Pr\bigl(M > m \,\big|\, G=g\bigr)
 \;=\; 1 \;-\; \prod_{v=1}^{k} \Phi\!\Bigl(\frac{m - g - \delta_v}{s}\Bigr).
\]
Each factor is strictly decreasing in $\delta_v$, hence the conditional tail
is nondecreasing in each $\delta_v$ separately, for every $g$ and every
$m$. Integrating over the law of $G$ preserves this coordinatewise
monotonicity, so over the orthant $\{\delta \le 0\}$ the tail is maximized at
$\delta = 0$, which is the equal-means configuration
$\mu_{p1}=\cdots=\mu_{pk}=\mu_q$. For $\rhow = 1$ ($s=0$) the conditional
probability is the indicator $\mathbf{1}\{\max_v \delta_v > m - g\}$, again
nondecreasing in each $\delta_v$, and the same argument applies. \qed

\emph{Remark.} The proof uses only that $q$ is a fixed comparator: the
selection $\max_v$ is over coordinates whose idiosyncratic noises are
conditionally i.i.d.\ given the shared factors. No coupling with a selected
runner-up is needed --- that difficulty belongs to estimands in which $q$ is
itself chosen by the selection event, which is outside $\Hprov$.

\subsection[Proof of Lemma 2]{Proof of Lemma~\ref{lem:mono} (monotonicity in $\rhow$)}

At the least-favorable configuration $\delta = 0$, consider the margin vector
$D_v = S_{pv} - S_q$, $v \le k$. It is centered Gaussian with
\begin{align*}
\operatorname{Var}(D_v)
&= (\rhow-\rhob) + (1-\rhow) + (1-\rhob)
 = 2(1-\rhob),\\
\operatorname{Cov}(D_u, D_v)
&= 1+\rhow-2\rhob \quad (u \ne v),
\end{align*}
so the variances are \emph{free of $\rhow$} while every off-diagonal
covariance is strictly increasing in $\rhow$; the implied correlation
$(1+\rhow-2\rhob)/(2-2\rhob)$ ranges over $[0, 1]$ for
$\rhow \in [2\rhob-1, 1]$ (and over $[\tfrac12, 1]$ on the narrower
factor-representation range $[\rhob, 1]$). Let $\rhow' > \rhow$ in
$[2\rhob-1, 1]$ and let $D, D'$ be the corresponding margin vectors ---
both valid exchangeable Gaussians on this domain, whether or not the
factor form \eqref{eq:twoblock} exists for them. By Slepian's inequality (equal means,
equal variances, $\operatorname{Cov}(D'_u,D'_v) \ge
\operatorname{Cov}(D_u,D_v)$ entrywise),
\[
\Pr\Bigl(\max_{v\le k} D'_v \le m\Bigr) \;\ge\;
\Pr\Bigl(\max_{v\le k} D_v \le m\Bigr)
\qquad \text{for all } m,
\]
i.e., $\Pr(M > m)$ is nonincreasing in $\rhow$. \qed

\emph{Remark (why this is the right monotonicity).} On the audited $z$
scale the standardization $\sqrt{2(1-\rhob)}$ is also free of $\rhow$, so
the same statement holds for $p_k(z;\rhow,\rhob)$, and reporting the
certificate at a lower bound $\rhow^-$ is conservative for any true
$\rhow \ge \rhow^-$.

\subsection[Proof of Proposition 1]{Proof of Proposition~\ref{prop:rhobmono} (monotonicity in
the comparator correlation)}
\label{app:rhobmono}

The certificate's statistic is standardized, and that changes what can
be proved. Divide the margins by their common standard deviation:
$X_v = D_v/\sqrt{2(1-\rhob)}$ has unit variance, and for $u \ne v$
\begin{equation}
\label{eq:rdef}
r(\rhow,\rhob) \;=\; \operatorname{Corr}(X_u, X_v)
\;=\; \frac{1+\rhow-2\rhob}{2(1-\rhob)}
\;=\; 1 - \frac{1-\rhow}{2(1-\rhob)} .
\end{equation}
The standardized vector is therefore \emph{equicorrelated with unit
variances}, so the tail depends on $(\rhow, \rhob)$ only through the
scalar $r$:
\[
p_k(z;\rhow,\rhob) \;=\; Q_k(z; r), \qquad
Q_k(z;r) = \mathbb{E}_H\Bigl[1 - \Phi\bigl((z-\sqrt{r}H)/
\sqrt{1-r}\bigr)^k\Bigr] .
\]
The pair $(\rhow,\rhob)$ is thus over-parameterized on the $z$ scale;
$r = 0$ recovers the \v{S}id\'ak law $1-\Phi(z)^k$ and $r = 1$ the
$k = 1$ law $1-\Phi(z)$. Slepian's inequality applied to two
equicorrelated unit-variance Gaussians gives $Q_k(z;\cdot)$
nonincreasing in $r$, while \eqref{eq:rdef} gives
\[
\frac{\partial r}{\partial \rhow} = \frac{1}{2(1-\rhob)} \ge 0,
\qquad
\frac{\partial r}{\partial \rhob} = -\,\frac{1-\rhow}{2(1-\rhob)^2}
\le 0 .
\]
Composing, $p_k$ is nonincreasing in $\rhow$ (Lemma~\ref{lem:mono}
again) and \emph{nondecreasing} in $\rhob$. It is constant in $\rhob$
exactly when $k = 1$ or $\rhow = 1$, and strictly increasing in the
interior otherwise, since Plackett's reduction formula makes
$\partial Q_k/\partial r$ strictly negative there. \qed

\emph{Remark (which scale the claim lives on).} At a \emph{fixed raw}
margin $m$ no such statement holds: changing $\rhob$ also changes the
margin variance, so the threshold moves in standard-deviation units and
Slepian's equal-variance hypothesis fails. The proposition is about the
standardized statistic $z = M/\sqrt{2(1-\rhob)}$ that the certificate
actually uses; raw-margin and standardized-margin monotonicity are
different statements, and only the latter is claimed.

\emph{Heterogeneous extension (the floor is a minimum, not a mean).} Let
the within-family block have arbitrary pairwise correlations $\rho_{uv}$
with unit variances, and let $\rhow^- = \min_{u \ne v} \rho_{uv} \ge
2\rhob - 1$ (the domain condition; both margin Gaussians below are then
valid). The margin vector then has
$\operatorname{Var}(D_v) = 2(1-\rhob)$ --- still free of the within-block
structure --- and $\operatorname{Cov}(D_u, D_v) = \rho_{uv} + 1 - 2\rhob$.
Setting $\rhow^- = \min_{u \ne v} \rho_{uv}$ gives entrywise covariance
domination of the heterogeneous vector over the equicorrelated one at
$\rhow^-$, with equal variances, so Slepian applies exactly as above:
$\Pr(M > m) \le \Pr(M_{\rhow^-} > m)$ for all $m$. Certification at
the minimum pairwise correlation is therefore valid for any
heterogeneous \emph{within-family} structure --- and the mean pairwise
correlation is \emph{not} a valid reporting point, which is why the
paper's floor is $\min_{\text{families}} \min_{u \ne v} \hat\rho_{uv}$.
One assumption is doing quiet work here: the comparator correlation is
held homogeneous ($\operatorname{Cov}(S_{pv}, S_q) = \rhob$ for every
$v$). If it is heterogeneous, $\operatorname{Var}(D_v) = 2 - 2\beta_v$
varies, Slepian's equal-variance hypothesis fails on the raw margins,
and this extension does not apply directly --- the mechanism
Appendix~\ref{app:masking} identifies behind the mean-column calibration
failures. Lemma~\ref{lem:robust}, proved next, repairs exactly this by
standardizing the margins first.

\subsection[Proof of Lemma 3]{Proof of Lemma~\ref{lem:robust} (comparator heterogeneity,
unequal variances, arbitrary selection)}
\label{app:robustproof}

Write $D_v = S_{pv} - S_q$, $\sigma_{D_v}^2 = \operatorname{Var}(D_v)
= \sigma_v^2 + \sigma_q^2 - 2\beta_v \sigma_v \sigma_q$, and let
$X_v = (D_v - \mathbb{E}D_v)/\sigma_{D_v}$ be the \emph{standardized}
margins: exactly standard Gaussian, with correlations
\[
r_{uv} \;=\; \operatorname{Corr}(D_u, D_v).
\]
\emph{Step 1 (selection bound).} For any selection rule $W$ taking
values in $\{1,\dots,k\}$,
$D_W/\sigma_{D_W} \le \max_{v \le k} D_v/\sigma_{D_v}$ pointwise, so
under $\Hprov$ ($\mathbb{E}D_v \le 0$),
\[
\Pr\!\left(\frac{D_W}{\sigma_{D_W}} > z\right)
\;\le\; \Pr\!\left(\max_{v\le k} \frac{D_v}{\sigma_{D_v}} > z\right)
\;\le\; \Pr\!\left(\max_{v\le k} X_v > z\right),
\]
the second inequality by the mean shift (Lemma~\ref{lem:lfc}'s coupling
argument, which needs no exchangeability). No assumption on \emph{how}
$W$ is chosen is used: selection on raw scores, on standardized scores,
or adversarially are all covered.

\emph{Step 2 (worst-case correlation floor).} With unit score
variances,
\[
r_{uv} \;=\; \frac{1 + \rho_{uv} - \beta_u - \beta_v}
{2\sqrt{(1-\beta_u)(1-\beta_v)}}
\;\ge\; f(\beta_u, \beta_v)
:= \frac{1 + \rhow^- - \beta_u - \beta_v}
{2\sqrt{(1-\beta_u)(1-\beta_v)}} .
\]
Substituting $\ell = 1-\beta^+$, $x = (1-\beta_u)/\ell \ge 1$,
$y = (1-\beta_v)/\ell \ge 1$, and $d = (1-\rhow^-)/\ell \ge 0$ gives
$f = (x + y - d)/(2\sqrt{xy})$ and, at the corner
$\beta_u = \beta_v = \beta^+$, $f = 1 - d/2$. The difference is the
identity
\[
f(\beta_u, \beta_v) - f(\beta^+, \beta^+)
\;=\; \frac{(\sqrt{x}-\sqrt{y})^2 + d\,(\sqrt{xy}-1)}{2\sqrt{xy}}
\;\ge\; 0,
\]
both terms being nonnegative for $x, y \ge 1$, so the global minimum
over any box $\beta_u, \beta_v \le \beta^+$ is exactly
\[
\min f \;=\; 1 - \frac{d}{2}
\;=\; \frac{1 + \rhow^- - 2\beta^+}{2 - 2\beta^+} \;=:\; r^- ,
\]
the homogeneous margin correlation at $(\rhow^-, \beta^+)$ ---
independent of $\beta^-$. The claim is about the \emph{margin} law:
$r^-$ is attained at that corner and the equicorrelated margin vector
at $r^-$ furnishes the upper bound. A full score-level covariance with
those parameters need not itself be feasible unless
$1 + (k-1)\rhow^- - k(\beta^+)^2 \ge 0$; the bound does not need it,
because the proof never leaves the margin scale.

\emph{Step 3 (Slepian on standardized margins).} The $X_v$ have equal
(unit) variances and equal (zero) means, so Slepian's inequality
compares them to the equicorrelated vector at $r^-$ whenever
$r^- \le r_{uv}$ entrywise (Step 2) and $r^- \ge 0$ (the
model-insufficient boundary $\rhow^- \ge 2\beta^+ - 1$, now stated at
$\beta^+$):
\[
\Pr\!\left(\max_v X_v > z\right) \;\le\; T(z; k, \rhow^-, \beta^+).
\]
Chaining Steps 1--3 proves the lemma. \qed

\begin{proposition}[Unconditional validity for the submitted model,
true standardization]
\label{prop:shipped}
Under the assumptions of Lemma~\ref{lem:robust}, let $W$ be the
submitted variant, chosen by any rule, and let
$Z_W = D_W/\sigma_{D_W}$ be its margin standardized by the \emph{true}
margin standard deviation. For every configuration of true means,
\[
\Pr\bigl(Z_W > c_\alpha(k;\rhow^-,\beta^+) \ \text{ and } \
\mu_{pW} \le \mu_q\bigr) \;\le\; \alpha .
\]
The critical value that tests $\Hprov$ therefore also bounds the
unconditional probability of certifying a submitted model that is not
truly better than the comparator. The statement is about $\sigma_{D_W}$,
not the estimated paired SE the audit plugs in; see the remark after the
proof.
\end{proposition}

\emph{Proof of Proposition~\ref{prop:shipped}.} Let
$N = \{v \le k : \mu_{pv} \le \mu_q\}$ be the variants that are not
truly better than the comparator. On the error event
$\{Z_W > c_\alpha\} \cap \{\mu_{pW} \le \mu_q\}$ the submitted
variant satisfies $W \in N$, so $N \ne \emptyset$ and
\[
Z_W \;=\; \frac{D_W}{\sigma_{D_W}}
\;\le\; \max_{v \in N} \frac{D_v}{\sigma_{D_v}} ,
\]
the inequality holding pointwise because $W$ is one of the indices in
$N$ and both sides use the \emph{same} (true) standardization. It would
\emph{not} hold with an estimated denominator on the left: a small
$\widehat{\operatorname{sd}}(D_W)$ can make the plug-in statistic
exceed every true-standardized margin. Every $v \in N$ has $\mathbb{E} D_v \le 0$, so the mean-shift
coupling of Lemma~\ref{lem:lfc} bounds the law of that maximum by its
centered counterpart, and enlarging the index set from $N$ to
$\{1,\dots,k\}$ can only increase a maximum. Steps 2--3 of the
preceding proof then apply verbatim to the centered standardized
margins, giving
$\Pr(\cdot) \le T(c_\alpha; k, \rhow^-, \beta^+) = \alpha$. \qed

\emph{Remark (what this does and does not give).} The bound is
\emph{unconditional}: it averages over which variant gets submitted.
It is not a statement conditional on the realized selection event, and
it does not identify \emph{which} variant is superior when the
provider-level null is rejected. It is also a \emph{known-$\sigma$}
theorem. The audit substitutes an estimated paired SE, which makes the
audited procedure a plug-in whose finite-sample behavior we assess
empirically (Appendix~\ref{app:masking}) rather than a case of this
proposition. ``Any selection rule'' permits arbitrary choice of
\emph{how $W$ is chosen}; it says nothing about a random denominator.
Note also that the guarantee requires no additional correction --- it is the same critical value --- because the null
subset $N$ is at most as large as the full family, and the
least-favorable configuration for a subset is dominated by that for the
whole.

\emph{Remark (bounded variance ratios).} If score variances are not
equal but bounded --- $\sigma_v/\sigma_q \in [1/R, R]$ --- Steps 1 and
3 are unchanged and only the correlation floor of Step 2 moves: with
$t_v = \sigma_v/\sigma_q$,
\[
r_{uv} \;=\;
\frac{\rho_{uv} t_u t_v + 1 - \beta_u t_u - \beta_v t_v}
{\sqrt{(t_u^2 + 1 - 2\beta_u t_u)(t_v^2 + 1 - 2\beta_v t_v)}},
\]
and $r^-(R)$ is the minimum over the compact box
$(\beta_u, \beta_v, t_u, t_v) \in [\beta^-,\beta^+]^2 \times
[1/R, R]^2$. An overstated floor would be anti-conservative, so a mesh
search will not do: the largest gradient \emph{seen on a mesh} is not a
bound on the gradient over the continuum, and a mesh minimum is an
upper bound on the minimum, not a lower one. We instead compute $r^-(R)$ by interval arithmetic with branch and
bound (\texttt{certified\_rmin}): every subbox has a rigorous
enclosure of $r$, subboxes whose lower enclosure exceeds the incumbent
are discarded, and what we report is the proven lower bound of the
remaining subboxes, outward-rounded by $10^{-9}$. This is
machine-checkable and needs no smoothness assumption. The
$\beta$-domain is the full $[-1, \beta^+]$, so no lower bound on the
hidden comparator correlations is smuggled in. At $R = 1$ the identity
above gives the answer in closed form; at $R > 1$ we report a certified
enclosure rather than an argmin claim, since branch and bound proves
$L \le \inf r \le U$ and finds the corner attaining $U$, which is not
quite a proof that the corner is the minimiser. For the audited rank-1
pair the enclosures are $r^- \in [0.4469138, 0.4469148]$ ($R = 1.1$),
$[0.4252691, 0.4252701]$ ($R = 1.16$) and $[0.3995626, 0.3995636]$
($R = 1.25$) --- width $10^{-6}$, the stopping tolerance --- against the
equal-variance value $0.4919$. The audit uses the certified lower
endpoints, and the run log (enclosure, remaining boxes, tolerance)
ships as \texttt{certified\_rmin\_log}. The effect is modest at realistic ratios: for the
audited rank-1 pair ($\rhow^- = 0.56$, $\beta^+ = 0.567$) the floor
moves from $0.492$ ($R = 1$) to $0.447$ ($R = 1.1$) and $0.400$
($R = 1.25$); the measured score-SD ratios $\sigma_v/\sigma_q$ in the
curated families have median worst-case $R = 1.04$ and maximum $1.16$
(\texttt{sigma\_ratios}), comfortably inside that range. Unbounded
variance asymmetry, by contrast, can push $r_{uv}$ below the $R = 1$
floor --- the bound is not variance-free, which is why $R$ is a stated
axis rather than an omitted one.

\emph{Remark (finite-$k$ margin domain).} The exchangeable margin
covariance (diagonal $a = 2-2\rhob$, off-diagonal $b = 1+\rhow-2\rhob$)
for $k\ge2$ is positive semidefinite iff $b \le a$ and $a + (k-1)b \ge 0$, i.e.
\[
2\rhob - 1 - \frac{2(1-\rhob)}{k-1} \;\le\; \rhow \;\le\; 1 ,
\]
which is strictly wider than the $[2\rhob-1, 1]$ on which the
certificate operates. On the gap --- negative equicorrelation $b < 0$
--- the selected-max tail \emph{exceeds} the \v{S}id\'ak limit and the
one-factor quadrature \eqref{eq:tail} does not represent it; floors
there are declared \emph{model-insufficient} rather than evaluated,
which is the conservative direction for the report but not a
certificate. A score-level representation with unit variances
additionally requires $1 + (k-1)\rhow - k\rhob^2 \ge 0$; none of the
margin-level results above need it.

\subsection{Selection-valid nuisance inference}
\label{app:svbb}

\begin{table}[!htbp]
\caption{Inference contract: each adjustment accounts for one more unobservable.}
\label{tab:contract}
\begin{center}
\footnotesize
\begin{tabular}{lll}
\toprule
adjustment & inputs and assumptions & guarantee \\
\midrule
reference tail & exchangeable Gaussian; known inputs & exact \\
heterogeneous margins & $(k, \rhow^-, \beta^+, R)$; true $\tau$ & conservative bound \\
selected-$\beta$ IF-SVBB & pair items; Fisher-$z$ UCB &
asymptotic bound; plug-in SE \\
joint $(\beta_W, \tau_W)$ & raw margin; joint IF region &
fixed-$K$ asymptotic \\
hidden-family axes & $\rhow^-, \Delta, R$ &
stated assumptions, never estimates \\
board-level selection & pair taken as fixed &
calibrated empirically only \\
\bottomrule
\end{tabular}
\end{center}
\end{table}

\emph{Regularity for the asymptotic corollary.} For each fixed $P$
the influence-function bound satisfies
$\max_{v \le K}\varepsilon_{v,n,P}(\eta) \to 0$ under: $G$ fixed,
$n_b/n \to \pi_b > 0$, bounded item outcomes, the variances
$V_v, V_q, \tau_v^2$ bounded away from zero, $|\beta_v| \le 1 -
\epsilon$, and i.i.d.\ items within independent benchmark strata. The statement is
pointwise in $P$; no uniformity over laws is claimed.

\emph{The empirical-Bernstein variant.} The finite-sample construction
bounds the five per-benchmark moments by empirical-Bernstein
concentration and combines them by interval arithmetic. With
$\eta = \gamma/2K$ split across six benchmarks the smallest
benchmark's moment pads reach $\approx 10\%$, the implied bound on
$\beta_W$ saturates at $1$, and every claim becomes
model-insufficient --- vacuous here, though this does not rule out
sharper finite-sample routes.

\emph{Effect on the audited board.} For the rank-1 claim the
fixed-pair bound $U_W(0.005) = 0.602$ becomes $U_W(0.005/27) = 0.615$,
and the certificate threshold moves from $\rhow^{*} = 0.689$ (plug-in)
to $0.756$ selection-valid, against $0.749$ for the fixed-pair
Berger--Boos version: accounting for selection in the correlation estimate
raises the required floor by $0.007$. The joint
$(\beta_W, \tau_W)$ version is more conservative, since worst-casing the
standard deviation lowers the effective margin from $z = 2.61$ to
$2.46$; it is the version that needs no plug-in studentization at all.
At $k = 27$ the required quantile is $1 - 0.005/27 = 0.99981$, which a
$B = 20{,}000$ bootstrap cannot resolve (about $3.7$ draws beyond it),
which is why the bound is analytic and the $10^5$-draw shared-row
stratified bootstrap below is its validation, not its source.

\begin{proposition}[Selection-valid nuisance bound]
\label{prop:selectedbeta}
Fix a data-generating law $P$, and suppose that for each \emph{fixed}
candidate $v$,
$\Pr_P\{\beta_v > U_{v,n}(\eta)\} \le \eta +
\varepsilon_{v,n,P}(\eta)$. Then for any selection rule $W$, possibly
depending on all candidate scores and the same items, and any
$k \le K$,
\[
\Pr_P\bigl\{\beta_W > U_{W,n}(\gamma/K)\bigr\}
\;\le\; \gamma + \sum_{v=1}^{k}\varepsilon_{v,n,P}(\gamma/K) :
\]
exact fixed-index bounds ($\varepsilon_{v,n,P} \equiv 0$) give an
exact selected bound; the influence-function bound gives
$\max_v \varepsilon_{v,n,P}(\eta) \to 0$ pointwise in $P$ under the
regularity above, with no uniformity over laws claimed. No
independence between candidates, or between selection and nuisance
estimation, is required.
\end{proposition}

\emph{Proof of Proposition~\ref{prop:selectedbeta}.} Fix $P$,
partition on the selected index, and use the fixed-index guarantee:
\begin{align*}
\Pr_P\bigl\{\beta_W > U_{W,n}(\gamma/K)\bigr\}
&= \sum_{v=1}^{k}
\Pr_P\bigl\{W = v,\ \beta_v > U_{v,n}(\gamma/K)\bigr\}\\
&\le \sum_{v=1}^{k}\Bigl[\frac{\gamma}{K}
+ \varepsilon_{v,n,P}\bigl(\tfrac{\gamma}{K}\bigr)\Bigr]\\
&\le \gamma
+ \sum_{v=1}^{k}\varepsilon_{v,n,P}\bigl(\tfrac{\gamma}{K}\bigr).
\end{align*}
Nothing about the law of $W$ enters, and the bounds may be arbitrarily
dependent. \qed

\emph{Sharpness of the $\gamma/k$ adjustment.} With disjoint failure
events $\{\beta_v > U_v\}$ and a rule that selects whichever
candidate's bound fails, the selected miss probability is exactly
$k\eta$; reducing the adjustment needs information a public auditor does
not have --- hidden losers' scores, selection gaps, or a prospective
holdout.

\emph{Finite reporting range.} Equation~(4) defines the ideal,
known-nuisance budget; it is not a claim that estimated nuisances permit
unrestricted inversion. All nuisance-adjusted curves in this paper use
$K_{\max}=100$ and recompute the caps at each integer $K\le100$.
We report $\widetilde{k}_{100}$ with $\ge100$ for a censored endpoint;
zero means no supported count in the reported range, and domain failure
is kept separate from non-rejection. The theorem applies for a fixed law
and fixed finite candidate bound; the finite reporting range does not
remove approximation error at the available sample size. In particular,
the $k=3$--$8$ family replays below do not validate coverage at $K=100$,
and reading several floors does not create a simultaneous confidence
statement over hidden-family assumptions.

\begin{table}[!htbp]
\centering
\caption{Finite-range nuisance-adjusted budgets for v1 rank 1, with
$K_{\max}=100$ and $(\Delta,R)=(0,1)$. Inputs and sampling assumptions
match Figure~\ref{fig:sensitivity}. The last row reaches the reporting
cap; it does not establish an unlimited budget.}
\label{tab:finite-budget}
\small
\begin{tabular}{ccc}
\toprule
Assumed floor & selected-$\beta$ budget & joint budget \\
\midrule
$0.42$ & 11 & 7 \\
$0.56$ & 13 & 8 \\
$0.707$ & 20 & 11 \\
$0.86$ & 82 & 30 \\
$0.99$ & $\ge100$ & $\ge100$ \\
\bottomrule
\end{tabular}
\end{table}

\emph{Growing multiplicity.} For $K = K_n$ growing with $n$,
asymptotic validity of the selected bound additionally needs
$K_n \max_v \varepsilon_{v,n,P}(\gamma/K_n) \to 0$; this is the
condition under which extending the sensitivity curve to large $K$ is
licensed.

\emph{Level-$\alpha$ validity of \eqref{eq:svbb}.} For true-standardized
$z$, fix $P$ and let
$E = \{\beta_W \le U_W(\gamma/K)\}$, so
$\Pr_P(E^c) \le \gamma + \sum_{v\le k}\varepsilon_{v,n,P}(\gamma/K)$
by the transfer theorem (at most $\gamma$ when the fixed-index bounds
are exact). On
$E$, the heterogeneity adjustment gives
$\max_v \beta_v \le \beta^+_K$, and the tail is nondecreasing in the
comparator cap (Proposition~\ref{prop:rhobmono}, extended to $\beta^+$
by Lemma~\ref{lem:robust}), so the tail evaluated at $\beta^+_K$ is at
least the true known-nuisance $p$-value. Hence under $\Hprov$
\[
\Pr_P\bigl\{p^{\mathrm{SVBB}}_K \le \alpha\bigr\}
\;\le\; \Pr_P(E^c) + \Pr_P\bigl\{p_{\mathrm{true}} \le \alpha - \gamma
\bigr\} \;\le\; \alpha +
\sum_{v\le k}\varepsilon_{v,n,P}(\gamma/K) ,
\]
at most $\alpha$ when the fixed-index bounds are exact. No
independence between the statistic, the selection rule and the nuisance
estimator is used. For the joint rectangle each side spends
$\gamma/2K$, so the remainder becomes
$\sum_{v\le k}[\varepsilon^{\beta}_{v,n,P}(\gamma/2K)
+ \varepsilon^{\tau}_{v,n,P}(\gamma/2K)]$.

\emph{Fixed-pair bound from the influence function.} With equal
benchmark weights $w_b = 1/G$ write $a_b = w_b^2/n_b$,
$\Gamma = \sum_b a_b \operatorname{Cov}_b(X_v, X_q)$,
$V_v = \sum_b a_b \operatorname{Var}_b(X_v)$ and likewise $V_q$, so
that $\beta_v = \Gamma/\sqrt{V_v V_q}$ and
$\tau_v^2 = V_v + V_q - 2\Gamma$. Writing
$\phi_{c,b} = (x-\mu_{vb})(y-\mu_{qb}) - \operatorname{Cov}_b$,
$\phi_{v,b} = (x-\mu_{vb})^2 - \operatorname{Var}_b(X_v)$ and
$\phi_{q,b}$ likewise, the benchmark-$b$ influence contributions are
\[
\psi_{\beta,b} = a_b\Bigl[\frac{\phi_{c,b}}{\sqrt{V_vV_q}}
- \frac{\beta_v}{2V_v}\phi_{v,b}
- \frac{\beta_v}{2V_q}\phi_{q,b}\Bigr],
\qquad
\psi_{\log\tau,b} = \frac{a_b}{2\tau_v^2}
\bigl[\phi_{v,b} + \phi_{q,b} - 2\phi_{c,b}\bigr],
\]
giving $\widehat{\operatorname{Var}}(\hat\beta_v)
= \sum_b \widehat{\operatorname{Var}}_b(\hat\psi_{\beta,b})/n_b$ and
the Fisher-$z$ bound
$U_v(\eta) = \tanh[\operatorname{atanh}(\hat\beta_v)
+ z_{1-\eta}\,\widehat{\operatorname{se}}(\hat\beta_v)/(1-\hat\beta_v^2)]$.
The formula can be evaluated at any $K$, but numerical evaluability
does not establish statistical validity. At our reporting cap $K=100$,
the selected-$\beta$ and each joint endpoint use tail probabilities
$5\times10^{-5}$ and $2.5\times10^{-5}$, respectively. These rely on
an extreme-tail normal approximation; the validation below is local
evidence, not a uniform coverage guarantee over that range. On the rank-1 pair, a $10^5$-draw item
bootstrap (resampling rows, since the outcomes are not all binary ---
TruthfulQA's mc2 is continuous, so a $2\times2$ cell reduction would be
wrong) gives Fisher-scale SE $0.0206$ against the influence function's
$0.0206$, a ratio of $1.003$, with skewness $0.035$ and excess kurtosis
$0.022$: the Fisher-$z$ scale is where this statistic is normal, which
is exactly the extrapolation the bound needs. The two bounds agree
at quantiles estimable from the bootstrap, $0.6024$ vs $0.6017$ at $99.5\%$ and
$0.6162$ vs $0.6145$ at the selection-valid $1 - 0.005/27$ --- the
latter resting on $18$ draws, which is why we do not rely on it.

\emph{Joint inference for $(\beta_W, \tau_W)$.} Apply the same union bound to a fixed-pair region
$C_v(\eta)$ for $(\beta_v, \tau_v)$; then $C_W(\gamma/K)$ covers the
selected pair with probability at least $1-\gamma$ when the fixed-pair
regions are exact, with the Proposition~\ref{prop:selectedbeta} remainder
otherwise. Using the rectangle
$\{\beta_W \le U_\beta,\ \tau_W \le U_\tau\}$ (each side at
$\gamma/2K$) and the \emph{observed raw margin}
$m_{\mathrm{obs}} = S_{pW} - S_q$, the tail is increasing in both
coordinates --- larger $\tau$ lowers $m_{\mathrm{obs}}/\tau$, larger
$\beta$ raises the tail --- so
\[
p^{\text{joint-SVBB}}_K \;=\;
\min\Bigl\{1,\;
\gamma + Q_K\Bigl(\frac{m_{\mathrm{obs}}}{U_\tau};\;
r^-\bigl(R;\rhow^-, U_\beta + \Delta\bigr)\Bigr)\Bigr\}.
\]
with $\widehat{\operatorname{Var}}(\widehat{\log\tau})
= \sum_b \widehat{\operatorname{Var}}_b(\hat\psi_{\log\tau,b})/n_b$
and $U_\tau(\eta) = \hat\tau\exp[z_{1-\eta}\,
\widehat{\operatorname{se}}(\widehat{\log\tau})]$. The display uses
$m_{\mathrm{obs}} \ge 0$, which adjacent-rank margins satisfy by
construction; for a negative margin the tail moves the other way in
$\tau$ and the region's lower endpoint would be the relevant one. No
statistic is standardized by an estimate here: the margin enters raw and
the standard deviation is worst-cased inside the region, which removes
the plug-in studentization Proposition~\ref{prop:shipped} had to assume
away. In the implementation $m_{\mathrm{obs}}$ is the composite-score
difference itself, never $z \times \hat\tau$: reconstructing it that way
would silently import whichever variance convention produced the
external $z$, and the joint statistic is supposed to depend on no such
convention at all. The check is cheap and worth running --- on every
audited pair $m_{\mathrm{obs}}/\hat\tau$ reproduces the published $z$ to
machine precision, which is what says the two pipelines agree. Whether the result is finite-sample or asymptotic is inherited
from the fixed-pair region, exactly as Proposition~\ref{prop:selectedbeta}
says --- a theorem end to end with the concentration region, and a
theorem up to the per-law remainders $\varepsilon_{v,n,P}$ with the
influence-function one.

\emph{Does the correction change coverage in practice?} The theory says a
fixed-pair bound need not cover a selected coordinate; whether it
actually fails is an empirical question, and on fully observed families
we can replay the selection to find out. Treating each curated family's
item matrix as the generating law and its plug-in $\beta_v$ as
population values, we resample items (shared indices across columns),
re-select $W^\ast = \arg\max_v S_v^\ast$ as a provider would, build both
bounds \emph{from the selected pair alone}, and record whether each
covers $\beta_{W^\ast}$. Against a nominal $99.5\%$, the ordinary
fixed-pair bound lands at $99.32$--$99.5\%$ (median $99.45\%$; below
nominal in $7$ of $12$ families, though no single family's shortfall is
resolved under simultaneous Clopper--Pearson bands at $B = 4{,}000$),
while the selection-valid bound covers $99.75$--$99.9\%$ in all twelve,
and the joint rectangle --- which must cover $\beta_{W^\ast}$ and
$\tau_{W^\ast}$ simultaneously, each side spending $\gamma/2k$ ---
covers $99.80$--$99.9\%$, also in all twelve. The measured leak is small
--- these families have $k = 3$--$8$, far from the $k = 27$ the theory
worst-cases --- but it is in the predicted direction, and the correction
removes it with a quantified adjustment.

\emph{When the operator can see everything.} A leaderboard operator
holding all $k$ candidate matrices can replace the union bound by a
joint max-$t$ band: bootstrap all candidates on shared item indices,
take the quantile of $\max_v (\theta_v - \hat\theta_v)/\hat s_v$, and
use $\tanh(\hat\theta_v + c_{1-\gamma}\hat s_v)$. Sibling nuisance
estimators are strongly positively correlated, so that band is
typically narrower than $\gamma/k$. The tool exposes both modes ---
public retrospective and operator full-disclosure --- because they
differ only in what data the auditor is given.

\subsection[Clone boundary: sanity check]{Clone boundary ($\rhow \to 1$): sanity check}

At $\rhow = 1$ the idiosyncratic terms vanish and
$S_{pv} = \sqrt{\rhob}\,Z_0 + \sqrt{1-\rhob}\,F_p + \mu_{pv}$: variants
differ by deterministic mean shifts only. Then
$M = \max_v \delta_v + G$ with $G \sim N(0, 2-2\rhob)$: the law of the
observed margin no longer depends on $k$ (hidden multiplicity is exactly
unidentifiable), and under $\Hprov$ ($\max_v \delta_v \le 0$) its null tail
is bounded by the $k=1$ law (hidden multiplicity is exactly harmless).
Equation~\eqref{eq:tail} reproduces this continuously:
$c_\alpha(k;\rhow,\rhob) \to c_\alpha(1)$ as $\rhow \to 1$ (numerically,
$c_{.05}(27; 0.999999, 0.5) = 1.6467$ vs.\ $z_{.95} = 1.6449$).

\subsection[Shipped-model claim and proof]{Shipped-model claim (Remark~\ref{rem:monotone}) and proof}

\begin{remark}[Limits of submitted-model certification]
\label{rem:monotone}
Within the class of nonrandomized threshold tests that are monotone in the
observed margin $z$ \emph{alone}, size control under $\Hprov$ forces the
critical value above the least-favorable quantile, so for $k > \kbar$ no
test in this class certifies the submitted-model claim either.
Tests outside the class (using the runner-up's item-level profile, the board's
full covariance, or family-membership side information) are not covered.
\end{remark}

Let $T$ reject the null iff $z > c$, a nonrandomized monotone threshold test
at hidden multiplicity $k$. Size control under $\Hprov$ requires
$\sup_{\delta \le 0} \Pr_\delta(z > c) \le \alpha$; by
Lemma~\ref{lem:lfc} the supremum is attained at $\delta = 0$, so
$p_k(c) \le \alpha$, i.e., $c \ge c_\alpha(k)$. If $k > \kbar(z_{\mathrm{obs}})$
then by definition of $\kbar$, $p_k(z_{\mathrm{obs}}) \ge \alpha$, hence
$z_{\mathrm{obs}} \le c_\alpha(k) \le c$: $T$ does not reject. Since the
submitted-model claim implies the provider-level claim (if the submitted
variant beats $q$, some variant does), no test in this class certifies
either claim when $k > \kbar$. \qed

\section{The full numerical grid and quadrature validation}
\label{app:grid}

\begin{table}[!htbp]
\caption{Critical values (standardized margin scale, $\alpha=.05$,
$\rhob = 0.50$ throughout) at the documented $k=27$, and effective
multiplicity. Numbers from the quadrature \eqref{eq:tail}
(4001-point rule; refinement to 20001 points changes no entry by more than
$10^{-6}$, and $4{\times}10^6$-draw Monte Carlo agrees within 1.2 SE;
Appendix~\ref{app:grid}).}
\label{tab:neff}
\begin{center}
\begin{tabular}{lccccccc}
\toprule
 & naive & Dunnett & Bonf. & exact (.42) & exact (.56) & exact (.707) & exact (.86) \\
\midrule
$c_{.05}$ & 1.645 & 2.725 & 2.902 & 2.77 & 2.68 & 2.54 & 2.31 \\
$\neff(27)$ & 1 & --- & 27 & 18.5 & 14.0 & 9.2 & 4.8 \\
\bottomrule
\end{tabular}
\end{center}
\end{table}

\textbf{Quadrature.} The tail \eqref{eq:tail} is computed by a 4001-point
trapezoid rule on $[-8.5, 8.5]$ against the Gaussian weight. Refinement to
a 20001-point rule on $[-10, 10]$ changes no reported cell by more than
$10^{-6}$. Monte Carlo validation across both runs: within
$2.1$ SE at $2 \times 10^6$ draws and within $1.2$ SE at
$4 \times 10^6$ draws in every validation cell, including the near-clone
corner --- consistent with an exact rule and shrinking Monte-Carlo noise.
(``Dunnett'' throughout denotes the shared-comparator equicorrelation-1/2
limit $\rhow = \rhob$; its critical value is free of $\rhob$.) (Near the clone boundary the integrand approaches a step, so a
\emph{fixed low-order} Gauss--Hermite rule is inaccurate --- a
200-node rule errs by $7$ Monte-Carlo SEs at
$(z, k, \rhow, \rhob) = (2.5, 27, 0.99, 0.5)$ --- which is why we use
the fixed fine grid.) Clone-boundary check:
$c_{.05}(27; \rhow \to 1, 0.5) = 1.6467$ vs.\
$z_{.95} = 1.6449$.

\textbf{Selected grid (excerpt; $\alpha = .05$, $z$-scale critical
values).} Full machine-readable tables ship with the code release.

\begin{center}
\begin{tabular}{ccccccc}
\toprule
$k$ & $\rhob$ & indep & Bonf & Dunnett & exact $\rhow{=}.7$ & exact $\rhow{=}.9$ \\
\midrule
10  & .25 & 2.568 & 2.576 & 2.448 & 2.224 & 2.006 \\
10  & .50 & 2.568 & 2.576 & 2.448 & 2.322 & 2.077 \\
27  & .25 & 2.895 & 2.902 & 2.725 & 2.413 & 2.120 \\
27  & .50 & 2.895 & 2.902 & 2.725 & 2.547 & 2.215 \\
100 & .25 & 3.283 & 3.291 & 3.047 & 2.628 & 2.247 \\
100 & .50 & 3.283 & 3.291 & 3.047 & 2.806 & 2.370 \\
\bottomrule
\end{tabular}
\end{center}

\section{Measurement details}
\label{app:measurement}

\textbf{Comparator-heterogeneity calibration and the correlation
ordering (detail).} Lemma~\ref{lem:robust}'s
$\beta^+$ axis is calibrated on the same families: for each curated
family we compute the score-matched correlation $\beta_v$ of \emph{every}
member with the family's audit comparator. The spread
$\Delta_b = \max_v \beta_v - \min_v \beta_v$ has median $0.048$
(range $0.009$--$0.383$); the certificate-relevant excess
$\max_v \beta_v - \beta_W$ over the identified winner correlation has
median $0.010$ and maximum $0.121$, and in 4/12 families the winner
\emph{is} the most comparator-correlated member (excess $0$)
(\texttt{beta\_heterogeneity}); the controlled family shows the same scale
against external comparators ($\Delta_b$ $0.09$--$0.23$,
Appendix~\ref{app:dial}). \S\ref{sec:audit} incorporates these adjustments
into adjusted verdicts.

Together these measurements order family \emph{correlation at pooled
granularity}: base$\to$instruct pairs $0.44$--$0.65$ $<$ community
merges/LoRA $\approx 0.71$ $<$ version series $\approx 0.78$ $<$
controlled checkpoint series $\approx 0.90$. This ordering describes the observed public populations, but
\S\ref{sec:measurement} shows it does not transfer across score functionals --- the same
checkpoint family has correlation $0.90$ pooled and $0.556$ on its own audited
composite --- so we use it as context, never as a transport argument:
what enters any verdict is the score-matched floor stated as an
assumption.

\begin{figure}[!htbp]
\centering
\includegraphics[width=\linewidth]{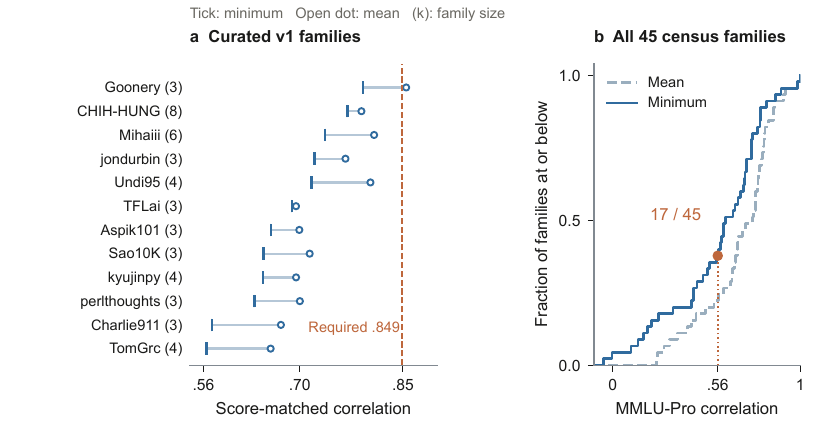}
\caption{\textbf{A descriptive mean is not a minimum-correlation guarantee.}
\textbf{(a)} Item-resampling score-matched minima and means for all twelve curated v1
families; parentheses give $k$, and full membership appears in
Appendix~\ref{app:measurement}. The orange line is the rank-1 joint
requirement at $k=27$. \textbf{(b)} Empirical cumulative distributions
of family means and minima among all 45 org-disjoint MMLU-Pro census
families; 17 minima fall below the $0.56$ sensitivity anchor.
These populations differ in tasks, construction rules, and family
sizes; their observations do not establish a hidden-family floor.
Controlled-family minima are shown separately in
Figure~\ref{fig:score-summary}c.}
\label{fig:family-evidence}
\end{figure}

\textbf{Census anchor detail (\S\ref{sec:measurement}).}
Restricting the census to $\ge$7B instruct bases leaves 6 families
(median mean $0.81$); their min-pairwise floor $0.42$ is a single
family's minimum pair, and dropping that family raises the restricted
floor to $0.705$, which would certify strictly more claims. Census
families are truncated at six members while curated families have
$k = 3$--$8$, so under any heterogeneity the observed minimum
mechanically falls as pairs are added --- small-family minima are
optimistic as large-family floors.

\textbf{Out-of-dataset and prompt-format detail (\S\ref{sec:measurement}).}
Correlated-Errors \citep{kim2025correlated}: v2-era models, 451 with
published accuracies; same-org same-base version series are 21 pairs
over 14 families. Holding the estimator tier as fixed as the data allow
(single-benchmark, pooled-item), the ordering remains: curated
merge/fine-tune families on MMLU alone give median $0.676$, the
census's submitter-declared families on MMLU-Pro $0.746$, and Kim's
version series on MMLU $0.778$ --- though this is one comparison, not a
de-confounding (\S\ref{sec:discussion}). PromptEval per-pair detail:
IQR widths $0.019$/$0.024$/$0.033$ for llama-3-8b/instruct ($0.646$),
mistral-7b-v0.1/instruct ($0.528$), and gemma-7b/it ($0.436$;
across-template range $0.27$--$0.48$): same-base correlations tend to
be higher on average, but the distributions overlap.

\textbf{Curation rules (deterministic).} A hand-curated family is the set
of matrix models satisfying all of: (i) same organization prefix (the
string before \texttt{\_\_} in the leaderboard identifier); (ii) an
explicit shared variant-series token in the model name (a common base-model
or series stem, e.g.\ \texttt{llama-2-13b-FINETUNE*},
\texttt{Pallas-0.*}, \texttt{Chupacabra-7B-v*}); (iii) family size
$k \ge 3$. Rules were frozen before any correlation was computed; no family
was added or removed afterward.

\par\noindent\begin{minipage}{\linewidth}
\textbf{Families (score-matched estimates; mean / min pairwise
$\hat\rhow$).}
\begin{center}
\small
\begin{tabular}{lccc}
\toprule
family (org $\times$ series) & $k$ & mean & min \\
\midrule
Goonery Huginn-13b            & 3 & .855 & .792 \\
Mihaiii Pallas (Yi-34B)       & 6 & .808 & .737 \\
Undi95 L2-13B                 & 4 & .803 & .717 \\
CHIH-HUNG llama2-13b          & 8 & .790 & .769 \\
jondurbin airoboros-7b        & 3 & .767 & .721 \\
Sao10K L2-13B                 & 3 & .714 & .648 \\
perlthoughts Chupacabra-7B    & 3 & .700 & .634 \\
Aspik101 llama2-13b-PL        & 3 & .699 & .658 \\
kyujinpy PlatYi-34B           & 4 & .695 & .647 \\
TFLai Platypus2-13B-QLora     & 3 & .695 & .689 \\
Charlie911 vicuna-7b-lora     & 3 & .673 & .572 \\
TomGrc FusionNet (SOLAR)      & 4 & .658 & .565 \\
\midrule
median of means / min of mins &   & .707 & .565 \\
\bottomrule
\end{tabular}
\end{center}
\end{minipage}\par

Full membership lists, per-family bootstrap CIs, the
accuracy-matching protocol for $\hat\rhob$, and the
PromptEval per-template distributions are in the released JSONs
(\texttt{d2\_results}, \texttt{r2\_functional\_rho},
\texttt{w1\_*}).

\begin{figure}[!htbp]
\begin{center}
\includegraphics{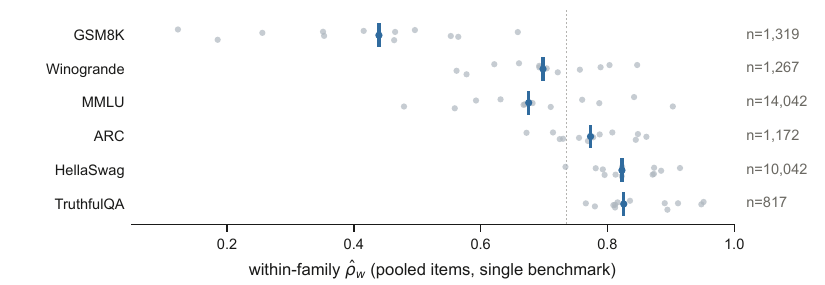}
\end{center}
\caption{Per-benchmark within-family $\hat\rhow$ (pooled items) across
the 12 curated families. Each gray dot is one family; blue tick and dot
mark the cross-family median. The benchmarks are ordered by median
$\hat\rhow$: GSM8K (median $.44$, $n = 1{,}319$ items) sits far below
MMLU ($.68$, $n = 14{,}042$). Because the equal-weight composite score's
sampling variance scales as $1/n_b$ per benchmark while a pooled-item
correlation weights by $n_b$, the composite behaves like the smallest
benchmark (GSM8K) while pooled estimates report the largest (MMLU) ---
the weighting mechanism described in \S\ref{sec:measurement}.}
\label{fig:bench_heterogeneity}
\end{figure}

\textbf{Ordering verification.} For v2, the audited ordering is official
by construction: models are ranked by the snapshot's
\texttt{average\_score} column, which is the leaderboard's published
average. For v1, the ordering is reconstructed from the per-item matrix as
the equal-weight mean of six benchmark accuracies; this matches the
official v1 formula: recomputing the official six-task average
from the official per-task metrics gives Kendall $\tau = 1.000$ over the
top 20 and 15/15 top-15 overlap with our reconstruction
(\texttt{v1\_ordering\_check.json}). What cannot be re-checked is the
retired page itself --- the v1 space was client-rendered, so web-archive
snapshots contain no table data; any residual discrepancy would have to
live in the display layer, not in the scores or the average the audit
uses.

\section{Subject-cluster sampling sensitivity}
\label{app:cluster}

The primary audit conditions on the benchmark design and resamples
individual items independently within each benchmark. As an alternative,
we treat MMLU's 57 subjects as exchangeable blocks and allow their
composition to vary. Other benchmarks retain shared-row item resampling.
This changes the sampling population; it is not a claim that subjects are
a random sample from all possible knowledge domains. It also leaves
unknown within-GSM8K template dependence and contamination unmodeled.

\textbf{Preserving the score and covariance.}
For subject $j$, let $n_j$ be its item count and $T_j$ the vector of
model-score sums. With $J=57$, $N=\sum_j n_j$, and
$\hat\mu=\sum_j T_j/N$, the cluster covariance of the original
item-weighted MMLU mean is
\[
\widehat C_{\mathrm{MMLU,cl}}
=\frac{J}{J-1}\frac{1}{N^2}
\sum_{j=1}^{J}(T_j-n_j\hat\mu)(T_j-n_j\hat\mu)^\top.
\]
We combine this with each other benchmark's item covariance divided by
its item count, using the squared original benchmark weights. The
bootstrap samples $J$ subjects with replacement, carrying all their
items, recomputes the ratio using the resampled total item count, and
recomputes the same cluster covariance. All models share the sampled
subjects and item rows. We never average subject means equally and never
center each subject separately: that would remove between-subject score
variation. For null replays we subtract each model's original
\emph{composite} mean from its bootstrapped score, preserving subject
contrasts before reselecting the winner.

\textbf{Controlled-family contrast.}
On the same five full-FT variants and 15{,}361 aligned items used in
Figure~\ref{fig:score-summary}, the pooled median is $0.894$.
The score-matched median is $0.459$ under item resampling but $0.921$
under subject resampling, with a percentile interval $[0.772,0.944]$
from $20{,}000$ shared resamples. The corresponding pairwise minima
are $0.417$ and $0.914$. MMLU's median share of composite variance rises
from $9.0\%$ to $86.7\%$ under subject sampling; common subject effects
now dominate the score covariance. The empirical bootstrap score
correlation ($0.918$ median) is close to the sandwich-based value. Thus the low composite correlation in the
headline example does \emph{not} survive allowing MMLU subject
composition to vary. The implication is that the correlation must match
both the score functional and the sampling population; pooling is not a
substitute for specifying either one.

\begin{table}[!htbp]
\centering
\caption{Sensitivity of the top three v1 claims to the MMLU sampling unit.
The observed margins and ranking are fixed. Floors refer only to the
Gaussian plug-in calculation at $k=27$, not the joint nuisance-adjusted
floor in Figure~\ref{fig:sensitivity}. A dash means even the Gaussian
single-test threshold is not passed.}
\label{tab:cluster}
\small
\begin{tabular}{llrrrr}
\toprule
Rank & MMLU sampling & $z$ & $\hat\rho_b$ & required floor & at $\rho_w^-=.56$ \\
\midrule
1 & items & 2.610 & .567 & .689 & not certified \\
1 & subjects & 2.532 & .759 & .861 & not certified \\
2 & items & 3.256 & .874 & .747 & model-insufficient \\
2 & subjects & 3.184 & .930 & .859 & model-insufficient \\
3 & items & 1.809 & .504 & .993 & not certified \\
3 & subjects & 1.610 & .703 & --- & not certified \\
\bottomrule
\end{tabular}
\end{table}

\textbf{Audit sensitivity.}
For rank 1, the cluster standard error is $0.004144$, compared with
$0.004020$ under item sampling; the empirical bootstrap SD is $0.004151$.
The comparator correlation changes much more than the standard error,
raising the plug-in required floor to $0.861$ (Table~\ref{tab:cluster}).
At $0.42$, the cluster-based floor is outside the margin domain
($0.42<2\times0.759-1$); at $0.56/0.71/0.86$ the claim is not
margin-certified. Rank 3 drops below the Gaussian naive threshold,
but its studentized bootstrap one-sided tail estimate is $0.0497$:
it is borderline and the two calibrations should not be conflated.
We do not transfer the item-level influence-function nuisance intervals
to 57 subject clusters; the cluster calculations are sampling-model
sensitivity diagnostics, not a new joint-coverage guarantee.

\textbf{Family replays and uncertainty.}
We also replay selection for all twelve curated families under this
sampling model ($B=20{,}000$ each), fixing the comparator and
recomputing the cluster-matched paired SE for the selected winner.
For reference we use the smallest correlation of the standardized
margin vector computed from the full-data cluster covariance; this
accommodates its observed covariance geometry, rather than importing
an item-based sibling floor. The Gaussian reference tail is at least
the empirical tail in $54/60$ threshold cells and in every cell of
$8/12$ families. None of the six pointwise exceptions is resolved using
two-sided Clopper--Pearson intervals simultaneous across all 60 cells.
These are empirical checks under the observed family configurations,
not a theorem for arbitrary subject dependence or hidden families.
The script \texttt{cluster\_bootstrap.py} includes checks that constant
but distinct subject blocks retain nonzero cluster uncertainty and that
singleton clusters recover the ordinary mean covariance. Results,
input hashes, and all cell-level intervals are in
\texttt{cluster\_bootstrap.json} (protocol \texttt{subject-ratio-v2}).

\section{Masking-validation details}
\label{app:masking}

\textbf{Comparator and board-level reselection replays.} The main-text
protocol fixes the comparator outside the loop, as in the model; only
the winner is re-selected. A variant that
\emph{also} re-selects the comparator as the best other-org model in
each replicate --- the protocol a real board implies --- is strictly
dominated in 12/12 families as well, by orders of magnitude (at
$B = 10^5$ the empirical tails are $\ge 300\times$ below the model
tail at every grid point): re-selecting the comparator can only lower
the winner's margin, so the fixed-comparator model tail is conservative
for it in the raw margin --- deterministically so, since the numerator
can only shrink. The studentized statistic is a different matter: a
different comparator also changes the denominator, so a smaller margin
can pair with a smaller SE. What we can say is that in this board's
replay the studentized tail fell too, which is calibration, not a
theorem. Going all the way --- replaying \emph{full board-level}
selection (winner re-selected over every model on the board, comparator
re-selected as the best other-org entry) under the global sharp null at
$B = 10^5$ --- drives the empirical tails to
$1.7\times10^{-3}$ at $z \ge 1.5$ and $0$ at $z \ge 2.5$: when both
sides of the margin are extreme order statistics of the same board,
margins shrink, and every certificate column dominates by orders of
magnitude. This result has two limitations. First, this is a
\emph{calibration result for the observed board's configuration under
the global sharp null}, not a uniform theorem over mean configurations:
the global sharp null is not least-favorable for board-level selection
(with two equal-mean leaders the selected margin is $|Z|$, doubling a
one-sided level, and crowding many equal means \emph{shrinks} top
spacings). Second, the per-claim certificate therefore remains
pair-conditional (\S\ref{sec:audit}); the replays calibrate, but do
not certify, the selection layers it conditions on. Item independence
is also a substantive sampling assumption; the separate subject-cluster
analysis in Appendix~\ref{app:cluster} changes both covariance and
standardization rather than reusing item-level error bars.

\textbf{Full-observability comparison.} As a second ground truth we
compare against the polyhedral conditional
test of \citet{lee2016exact} on each unmasked family.
The certificate, which sees only the
winner's margin, is more conservative than the full-observability
conditional $p$-value in 12/12 families, as it should be; under the
sharp null both tests reject at rates consistent with $\alpha = 0.05$
within Monte-Carlo error. In the small-benchmark calibration (4
families $\times$ six benchmarks, $B = 10^5$;
Figure~\ref{fig:calibration}) the min-pairwise column's worst cell is
within Monte-Carlo bands, attributable to the Gaussian approximation
and comparator-correlation heterogeneity at $n \approx 10^3$.

\textbf{Gaussian versus bootstrap tails on the audited composites.}
For each of the five audited top claims we center both correctness
columns per benchmark (sharp null) and draw $10^5$ stratified item
bootstraps of the studentized composite margin
(\texttt{gauss\_vs\_bootstrap}). The empirical upper tails track the
standard normal within Monte-Carlo resolution at every grid point ---
at $t = 2.5$, empirical $0.0061$--$0.0068$ against Gaussian $0.0062$
--- with $|\text{skew}| \le 0.013$ and excess kurtosis $\le 0.03$
across all five claims, including the two v2 composites whose variance
is dominated by generative benchmarks. The marginal Gaussian
approximation behind every audit column is not the binding assumption
at these sample sizes.

\begin{figure}[!htbp]
\centering
\includegraphics[width=\linewidth]{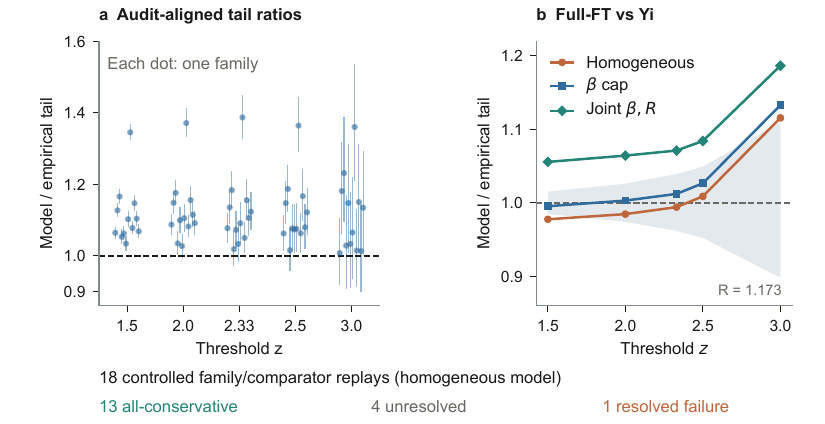}
\caption{\textbf{Calibration and its observed exception.}
\textbf{(a)} Minimum-based model/empirical tail ratios, twelve families
at each of five thresholds; small horizontal offsets separate families.
Vertical bars are the ratio bounds $q/U_{\rm emp}$ to $q/L_{\rm emp}$.
\textbf{(b)} Full-FT/Yi homogeneous, beta-cap, and joint beta/variance
predictions divided by the empirical tail. The controlled-replay
summary uses homogeneous predictions: 13 settings have all cells
resolved conservative, four have unresolved cells but no resolved
failure, and one has a resolved failure. Error bars and shading show
empirical Monte Carlo uncertainty ($B=10^5$), each direction
simultaneous at 95\% within a replay; they are not joint two-sided
95\% bands. The joint curve uses certified $r^-=0.56317237$ at
$(\rhow^-,\beta^+,R)=(0.417,0.23,1.173)$.}
\label{fig:validation}
\end{figure}

\textbf{Submission-history census.} The leaderboard's public request history
(10{,}019 submissions, 1{,}663 organizations) shows
multiplicity is the norm, not the exception: mean 6.0 submissions per
organization, 68\% submitted at least two models, 3.2\% at
least 27 --- the documented private $k=27$ sits at the 97th percentile of
even the \emph{public} per-organization counts. Submission metadata nominally includes base-model and deletion status,
but neither sharpens the census: submitter-declared \texttt{base\_model}
is non-empty in only 16.2\% of the history (\S\ref{sec:audit}), and the
history is append-only in practice --- 2 files in a 4{,}000-file sample
have \texttt{DELETED} status --- so organization-level public counts
remain descriptive counts, not lower bounds on the candidate set
behind any particular submitted model.

\begin{figure}[!htbp]
\begin{center}
\includegraphics{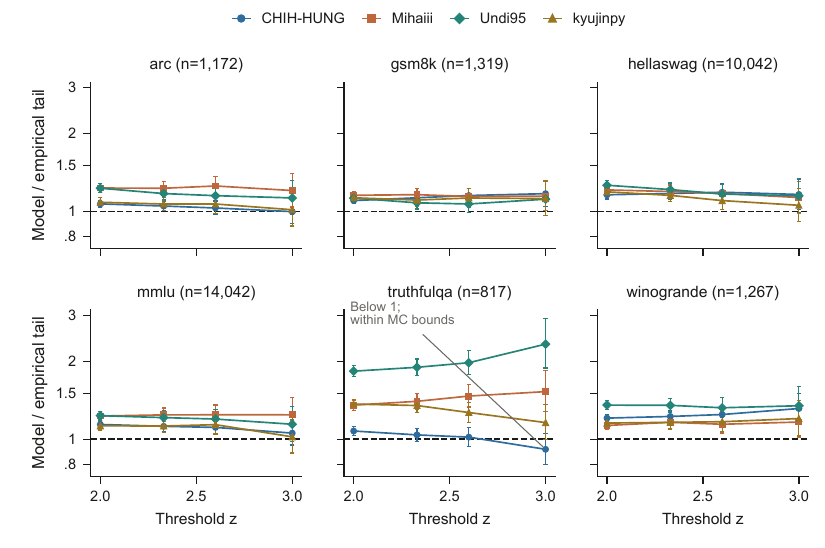}
\end{center}
\caption{Small-benchmark calibration at the min-pairwise
(Lemma-2-licensed) model column. For each benchmark and each of four
families, all $k{+}1$ correctness columns are centered (sharp null) and
the one-shot selection is replayed by item bootstrap within that
benchmark alone ($B = 10^5$); curves show the model/empirical
tail ratio. One of 96 point estimates is below one (ratio $0.915$),
but remains within Monte Carlo uncertainty; no min-column cell is
resolved anti-conservative. Bars transform empirical Clopper--Pearson
bounds as $q/U_{\rm emp}$ to $q/L_{\rm emp}$, with each direction
simultaneous at 95\% across all 96 cells. The mean-$\rhow$ column,
not shown, has four resolved failures. These are not joint two-sided
95\% intervals (Appendix~\ref{app:masking}).}
\label{fig:calibration}
\end{figure}

\textbf{Calibration detail with Monte-Carlo bands (per
Appendix~\ref{app:masking}).} At $B = 10^5$, pointwise
$p_{\mathrm{model}} \ge p_{\mathrm{emp}}$ holds at every grid $z$ in
12/12 families at both the family's own $\hat\rhow$ and the floor;
minimum per-cell ratio $1.013$. Pointwise domination alone does not
establish domination of the \emph{true} tails. For the aggregate
counts below, Clopper--Pearson bounds are Bonferroni-simultaneous
in each direction at level $0.05$ over the entire reported analysis
(60 or 96 cells), rather than separately within each family curve.
Figure~\ref{fig:calibration} displays the 96-cell bounds; the close-ups
in Figures~\ref{fig:validation} and~\ref{fig:validation-support} use
their explicitly stated per-replay scope.
The aggregate results (\texttt{mc\_calibration\_bands}) are: pooled at own-$\hat\rhow$
29/60 cells resolve conservative, 31 consistent, 0 resolve
anti-conservative; pooled at the floor 52/60/0; audit-aligned composite
(min column) 39/60/0; single-benchmark min column 78/96 resolve
conservative, 18 consistent, 0 resolve anti-conservative --- while the
mean-$\rhow$ column has 4/96 pointwise failures, \emph{all four
statistically resolved} (the $k{=}8$ family on TruthfulQA; min-column
worst cell ratio $0.915$ at $z=3$ stays within bands). The initial
$B = 2{,}000$ run's apparent failures (6/12 strict, worst ratio $0.58$
on 12 exceedance events) disappear at $50\times$ resolution.
Per-family ratio tables and the full-observability comparison
(certificate vs.\ polyhedral-conditional vs.\ winner-vs-pool-runner-up
$p$-values, and null rejection rates at 4{,}000 replications) ship
in \texttt{d3\_results\_B1e5}, \texttt{w2\_results\_B1e5},
\texttt{mc\_calibration\_bands}, and \texttt{w3\_results}.

\textbf{Census pipeline disclosure.} The org-disjoint census is the
45 \emph{largest} qualifying families (a download budget, not a
population); each family was truncated to its first six members
in metadata order --- fewer members means fewer pairs, so the
min-pairwise statistics reported are, if anything, \emph{optimistic}
relative to the full families; and 22 of 209 attempted
members (10.5\%) were dropped for missing or unscorable detail
repos, an attrition that is not random with respect to organization.
Public submission counts lower-bound organization-level public
probing wherever histories are used --- not any specific claim's
sibling-pool $k$ (Appendix~\ref{app:masking}).

\textbf{Bootstrap protocols.} All resampling is a nonparametric item
bootstrap, stratified so that each benchmark contributes exactly its own
$n_b$ items per replicate (and, where families are pooled, stratified by
family as well); correctness columns are centered to equal means before
resampling wherever a sharp null is imposed. Confidence intervals are
percentile intervals at 95\% ($B = 400$) for the family-level
$\hat\rhow$ and $\hat\rhob$ CIs of \S\ref{sec:measurement}. The
\emph{fixed-pair diagnostic} Berger--Boos adjustment uses one-sided $99.5\%$
Fisher-$z$ upper bounds ($B = 20{,}000$), matching the budget
$\alpha_\rho = 0.005$:
Proposition~\ref{prop:rhobmono} puts the supremum at the upper
endpoint, so a lower endpoint is never used, and a one-sided bound
leaves $\approx 100$ rather than $\approx 50$ draws beyond the quantile
--- the Fisher-$z$ scale keeps it inside $[-1,1]$ near the clone
corner. The \emph{primary} selection-valid and joint adjustments need the
$1 - \gamma/K$ and $1 - \gamma/2K$ quantiles, which that resolution
cannot see; they are built analytically from the influence function and
validated against the $10^5$-draw bootstrap of Appendix~\ref{app:svbb}. Selection nulls use $B = 10^5$ replicates (pooled and
single-benchmark masking, controlled-family coverage, comparator
reselection) except the full-observability null-rate comparison
($4{,}000$ replicates per family, Monte-Carlo SE $\approx 0.003$).

The controlled family's regression on training factors needs a separate protocol, because
its 190 outcomes form a complete dyadic graph rather than a sample:
each replicate resamples items within benchmark, applies the
\emph{same} indices to all 20 correctness columns, recomputes the full
$20\times20$ correlation matrix and refits ($B = 2{,}000$), so
shared-variant and shared-item dependence both remain in the
interval. Resampling dyads, or resampling each pair's items separately,
would destroy exactly the dependence that has to be preserved.
Seeds are fixed in the released scripts.

\section{Controlled families under item resampling}
\label{app:dial}

\textbf{Design.} All observational families are found; this one is made,
so its generating process --- the closest analogue of a provider's private
checkpoint series --- is fully known. We train 20 QLoRA variants
of a single base model (Qwen2.5-7B) with pre-specified factors: four
disjoint instruction-data mixtures (hash-partitioned shards of
tulu-v2-sft-mixture), LoRA rank $\in \{8, 64\}$, training steps
$\in \{500, 2000\}$, and seed replicates; every variant plus the base is
then evaluated per-item on MMLU and GSM8K (15{,}361 aligned
items).

\textbf{Controlled-family correlations.} Pairwise $\hat\rhow$ across the
190 variant pairs: $0.847$--$0.939$, median
$0.896$, higher than the pooled correlations in the observational
groups in \S\ref{sec:measurement}.

\textbf{Training-factor analysis.} Regressing pairwise $\hat\rhow$ on differences in training settings
over the 190 pairs. \emph{The 190 dyads are not 190 independent
observations}: each variant occurs in 19 of them and every correlation is
estimated from the same aligned items, so i.i.d.\ OLS standard errors
are not usable and we do not report them. We keep the OLS coefficients
as finite-network projections and attach three uncertainties, each
answering a different question --- resampling items, dropping a
variant, and relabelling the training settings:

\begin{center}
\small
\begin{tabular}{lcccc}
\toprule
factor & $\hat\beta$ & item-bootstrap 95\% CI & leave-one-variant
range & perm.\ $p$ \\
\midrule
intercept             & $+0.9035$ & $[+0.899, +0.908]$ & $[+0.901, +0.907]$ & --- \\
same data mixture     & $+0.0100$ & $[+0.0088, +0.0113]$ & $[+0.0088, +0.0127]$ & $.0005$ \\
seed-replicate only   & $-0.0108$ & $[-0.0126, -0.0088]$ & $[-0.0130, -0.0084]$ & $.029$ \\
different LoRA rank   & $-0.0218$ & $[-0.0236, -0.0200]$ & $[-0.0252, -0.0194]$ & $.0005$ \\
$\Delta$steps (/1500) & $+0.0030$ & $[+0.0010, +0.0051]$ & $[+0.0013, +0.0045]$ & $.573$ \\
\bottomrule
\end{tabular}
\end{center}

The bootstrap resamples items within benchmark and applies the
\emph{same} indices to all 20 correctness columns, recomputes the full
$20 \times 20$ correlation matrix and refits ($B = 2{,}000$); this
preserves both the shared-variant and the shared-item dependence that
an i.i.d.\ fit ignores. Its intervals are narrow \emph{by construction}: it holds the 20
trained variants fixed and asks only about item-sampling precision, so
it must not be read as a statement about model-level significance. That
question belongs to the other two columns --- the leave-one-variant
range, and an MRQAP permutation test with double semi-partialling
($2{,}000$ node relabellings of the residual matrix, which moves labels
without disturbing the dyadic dependence), which has limited power with
only $20$ nodes. Read together: sharing training data is associated
with higher correlation and mismatched adapter capacity with lower,
with permutation $p=.0005$ for each association. The seed
term differs: pairs differing \emph{only} by seed have correlations
$0.011$ \emph{below} what the additive model predicts for
same-mixture pairs, and that shows up in every column ($p = .029$,
though not after Bonferroni correction across the four factors).
The step-count coefficient has an item-bootstrap interval above zero,
but the permutation test does not reject ($p=.573$; Figure~\ref{fig:c6strata}). A dyadic cluster-robust sandwich
\citep{aronow2015cluster} --- which allows shared-endpoint dependence
but still treats disjoint dyads as independent, and so does not cover
the shared-item layer --- is reported alongside in
\texttt{dyadic\_inference} for comparison. Nothing here is a
randomization statement: training settings were pre-specified, not
randomized, so these are associations within one controlled family.

\begin{figure}[!htbp]
\begin{center}
\includegraphics{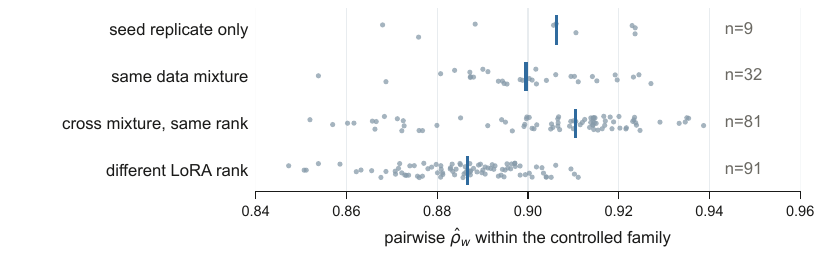}
\end{center}
\caption{The 190 pairwise correlations of the controlled family
by training-factor stratum (gray dots; blue tick = stratum median).
Pooled correlations range from approximately $0.85$ to $0.94$ across all strata.}
\label{fig:c6strata}
\end{figure}

\textbf{When variance ratios affect coverage.} A second
controlled family --- 10 variants of Qwen2.5-1.5B, five full
fine-tuned and five LoRA, same shards, steps and seeds --- reproduces
the difference between pooled and score-matched correlations (pooled $\hat\rhow$ median $0.89$/$0.85$ for the
full-FT and LoRA strata, composite $0.46$/$0.42$) and, unlike the 7B
family, exposes a case where $\beta^+$ alone does not give pointwise
domination over the saved empirical grid. Against an
external comparator (Yi-1.5-9B-Chat) the full-FT stratum's empirical
selection null crosses the homogeneous tail, and at $B = 10^5$ the
crossing is \emph{statistically resolved}: the $z = 1.5$ cell exceeds
it outside simultaneous Clopper--Pearson bands ($0.1892$ against a
homogeneous $0.1850$). The $\beta^+$-robust tail is never resolved
anti-conservative --- it dominates pointwise at four of the five grid
points and is grazed within bands at the fifth. The mechanism is the one
Lemma~\ref{lem:robust}'s second axis anticipates: small-model
fine-tuning moves GSM8K accuracy from $0.63$ to $0.14$--$0.41$, so the
variants differ from the comparator in composite score SD, which the
equal-variance reading ignores. The ratio that matters is the
theorem's, $R = \max_v \max(\sigma_v/\sigma_q, \sigma_q/\sigma_v)$,
computed against \emph{that} comparator: $R = 1.17$ for this stratum
against Yi-1.5-9B-Chat (the family reaches $1.34$ against its own base,
and $1.31$ for the LoRA stratum against zephyr --- the ratio is a
property of the pair, not of the family). Evaluating the \emph{joint}
$(\beta^+, R)$ column at that certified floor
($r^- \ge 0.563$, against $0.621$ under equal variances) restores
domination in all five cells.
The variance-aware construction restores pointwise domination in
this replay; the beta-only cell within the Monte Carlo bounds does
not itself establish a population-level failure. Operators should
report the variance axis when family members differ materially in
score precision.

\begin{figure}[!htb]
\centering
\includegraphics[width=\linewidth]{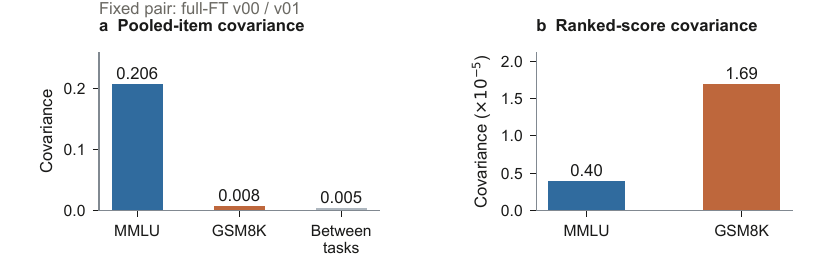}
\caption{\textbf{Exact covariance decomposition for the worked v00/v01 pair.}
\textbf{(a)} Pooled-item covariance is the sum of within-MMLU,
within-GSM8K and between-task mean components.
\textbf{(b)} Independent within-task resampling gives the composite
score covariance, with no between-task term. The units differ
between panels. Correlations require each covariance's own marginal
variance normalization; the task correlations cannot generally be
averaged with one common set of weights. All moments use the same
15{,}361 aligned items and the canonical ddof-zero convention.}
\label{fig:score-decomposition}
\end{figure}

\begin{figure}[!htb]
\centering
\includegraphics[width=\linewidth]{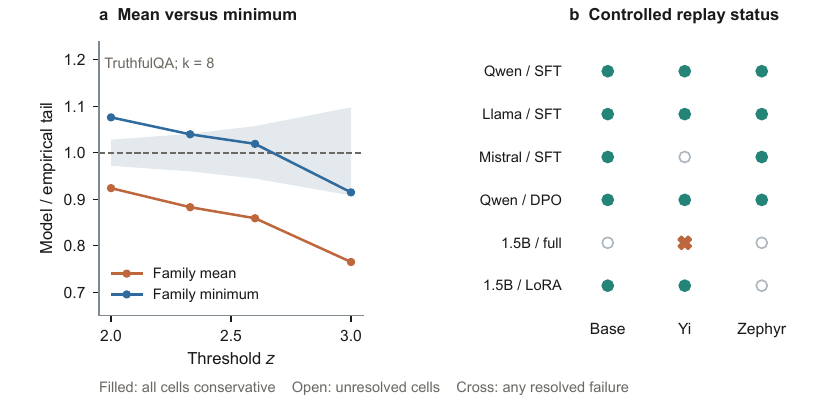}
\caption{\textbf{Why minima and unresolved validation cells matter.}
\textbf{(a)} TruthfulQA, CHIH-HUNG ($k=8$): the family mean gives four
resolved anti-conservative cells. The minimum has no resolved
anti-conservative cell, although at $z=3$ its point estimate lies
below the empirical tail. Gray: empirical Clopper--Pearson bounds,
each direction simultaneous at 95\% across these four thresholds.
\textbf{(b)} Homogeneous controlled replays, each over five thresholds:
filled circles mean every cell is resolved conservative; crosses
mean at least one resolved anti-conservative cell; open circles
mean otherwise. The latter are not demonstrated successes or
failures. Each replay uses $B=10^5$; shared families and comparators
are not independent studies.}
\label{fig:validation-support}
\end{figure}

\par
\begin{samepage}
\textbf{Accuracy and parsing controls.} Small-model fine-tuning reduces
GSM8K accuracy. We therefore check whether low-scoring variants or
unparseable answers account for the correlation difference using five
analyses of the 10-variant pool. \emph{(i) Scoring rule}: strict-match,
which also counts formatting mismatches as errors, lowers the \emph{composite}
correlation further ($0.252$ vs $0.294$) while leaving the pooled one
essentially unchanged ($0.858$ vs $0.846$) --- so the gap widens rather
than closes, and it is not a lenient-parser artifact.\par
\end{samepage}

\emph{(ii) Parse rates}: under
flexible-extract every variant still emits a readable answer on
$80$--$100\%$ of GSM8K items, so the decorrelation is in \emph{which}
items are right, not in whether anything parses. \emph{(iii) Best
variants only}: restricting to the top five by composite score raises
the median from $0.294$ to $0.426$ --- weak variants do contribute ---
but remains below half the same top-five subset's pooled $0.885$. \emph{(iv)
Accuracy-matched pairs}: among the $16$ pairs within $0.02$ composite
accuracy of each other, the median is $0.444$, so unequal marginal
success rates do not explain the gap either. \mbox{\emph{(v) Combined-pool replay}}: treating all ten variants as one candidate pool
($\rhow^- = 0.170$ on the composite) and replaying selection against
each comparator, the robust certificate dominates the empirical null in
every cell for all three. The interpretation supported by all five checks is the
structural one: the composite's variance is dominated by the small,
generative benchmark on which these siblings genuinely disagree.

\begin{figure}[!htbp]
\centering
\includegraphics[width=\linewidth]{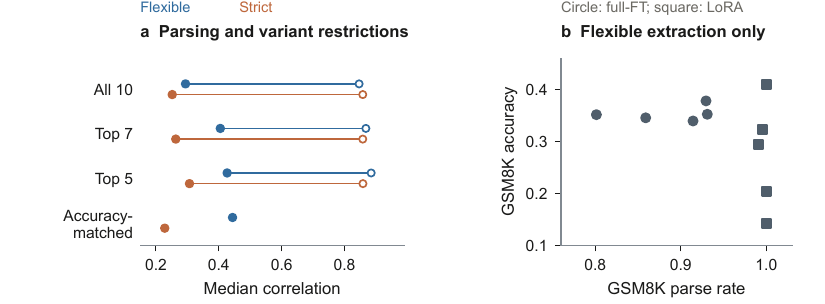}
\caption{\textbf{The score-matching gap remains after descriptive controls.}
\textbf{(a)} Medians among all ten Qwen2.5-1.5B variants, the top seven
or five by the corresponding composite score, and pairs within
$0.02$ composite accuracy. Filled points are score-matched; open
points are pooled. The accuracy-matched summaries have no saved
pooled estimate, so no endpoint is imputed. Top subsets and matched
pairs can differ between parsing rules; these are not causal parser
effects. \textbf{(b)} Per-variant GSM8K accuracy and parse rate, all under flexible
extraction. Neutral circles/squares distinguish full-FT/LoRA;
parser colors apply only to (a).
A high parse rate need not imply agreement about which answers are
correct. These are ten variants from one controlled setup, not a
representative sample of hidden frontier-model searches.}
\label{fig:controls}
\end{figure}

\textbf{External-comparator calibration.} The
base-comparator coverage check below runs in the model's
\emph{degenerate corner}: QLoRA moves the model so little that the
base-model comparator is as correlated with the variants as they are
with each other ($\hat\rhob \approx \hat\rhow \approx 0.90$ at
pooled-item granularity), exercising the shared-comparator limit. We
therefore evaluated two \emph{external, org-disjoint} comparators
(Yi-1.5-9B-Chat, zephyr-7b-beta) on the same MMLU+GSM8K items. On the
equal-weight two-task composite --- whose score-matched convention drops
the family's $\rhow$ to $\min 0.418$ / median $0.556$, because the
generative GSM8K task decorrelates the variants (another instance of
\S\ref{sec:measurement}'s estimator dependence) --- the external
comparators sit at $\beta$ mean $0.372$/$0.213$ (max $0.416$/$0.244$):
genuinely $\rhow > \rhob$. Replaying one-shot selection at
$B = 10^5$: the certificate tail dominates the empirical selection null
in all $15$ cells (three comparators including base), under both the
homogeneous column and the Lemma~\ref{lem:robust} robust column at
$\beta^+ = \max_v \beta_v$, with every cell \emph{statistically
resolved} conservative under simultaneous Clopper--Pearson bands
(\texttt{c6\_external\_comparators}, \texttt{mc\_calibration\_bands}).
Across the pre-specified training settings, pairwise correlations differ
by at most $\approx 0.09$ against an intercept of $0.90$.
Pooled item-error correlations remain high throughout this QLoRA family (the
seed-replicate stratum, $n = 9$ pairs, is a power-limited null).

\textbf{Certificate coverage with known ground truth.} Treating the 20
variants as the hidden pool ($k = 20$) and the base model as the
comparator, centering all columns and replaying one-shot selection by item
bootstrap ($B = 10^5$; an initial $B = 2{,}000$ run showed one
apparent sub-empirical cell at $z = 1.5$ that vanished at $50\times$ the
resolution --- Monte-Carlo noise, the same pattern as
Appendix~\ref{app:masking}): the certificate tail lies at or above the
empirical selection null at every tested margin (3 of 5 cells resolved
conservative under simultaneous binomial bands, none resolved
anti-conservative)
($z \in \{1.5, 2, 2.33, 2.6, 3\}$; e.g.\ $0.0559$ vs
$0.0475$ at $z = 2.6$) --- on a family whose generating process is
fully known, the certificate is pointwise conservative at the
simulated resolution; three of five cells are statistically resolved
conservative and none is resolved anti-conservative. (Here $\hat\rhob \approx \hat\rhow
\approx 0.90$: QLoRA moves the model little, so the base is
nearly as correlated with the variants as they are with each other ---
the near-clone case.)

\textbf{Across architectures and update rules.} The same protocol ---
same training-factor grid, same per-item MMLU$+$GSM8K evaluation, same
selection replay at $B = 10^5$ --- run on two further lineages and one
further learning rule (\texttt{c6x\_results}):

\begin{center}
\small
\begin{tabular}{llccccc}
\toprule
base & rule & $k$ & \multicolumn{2}{c}{score-matched $\hat\rhow$} &
pooled & robust replay \\
& & & min & median & median & (of 3 comparators) \\
\midrule
Llama-3.1-8B    & QLoRA SFT & 20 & .45 & .58 & .78 & dominates 3/3 \\
Mistral-7B-v0.3 & QLoRA SFT & 20 & .44 & .52 & .73 & dominates 3/3 \\
Qwen2.5-7B      & DPO       & 10 & .72 & .81 & .96 & dominates 3/3 \\
Qwen2.5-7B      & QLoRA SFT & 20 & .42 & .56 & .90 & dominates 3/3 \\
Qwen2.5-1.5B    & full FT   &  5 & .42 & .46 & .89 & 2/3 (the $R$ cell) \\
Qwen2.5-1.5B    & LoRA SFT  &  5 & .28 & .42 & .85 & dominates 3/3 \\
\bottomrule
\end{tabular}
\end{center}

Three observations apply to these six strata. \emph{First}, the pooled-vs-score-matched gap is in
every row: it is a property of the composite score functional, not of
any one lineage or recipe. \emph{Second}, the update rule moves the
correlations more than the architecture does: the two new SFT lineages have
score-matched medians of $0.52$--$0.58$, compared with $0.56$ for the original
Qwen family and $0.81$ for DPO on the same base. In this
controlled family, preference optimization preserves substantially more
of the base's error structure than any SFT stratum does, so a floor
calibrated on SFT families
is conservative for the measured DPO siblings, and an operator
auditing a mixed board should state which correlation case the floor describes.
\emph{Third}, coverage: of the $18$ family-stratum$\times$comparator
replays ($90$ grid cells), the only resolved anti-conservative cell
is the full-FT equal-variance cell examined above, and the
$\beta^+$-robust column dominates pointwise in $17$ of $18$ replays ---
accounting for unequal variances restores domination in the remaining replay.

\textbf{Extension-family training protocol.} The two SFT extension
lineages share the original training-factor grid (four hash-disjoint data shards
$\times$ adapter capacity $\times$ steps $\times$ seed); the DPO and
full-FT strata use separately pre-specified designs but share the same
per-item evaluation and selection-replay protocol (TRL training,
\texttt{lm-eval} per item on MMLU$+$GSM8K, bf16,
\texttt{--log\_samples}, the pinned flexible-extract filter). All $60$
scheduled variants trained, and all $60$ plus the $3$ base checkpoints
evaluated with none missing; the loader rejects any run missing either
task rather than silently analyzing the remainder:

\begin{center}
\small
\resizebox{\linewidth}{!}{%
\begin{tabular}{lllllc}
\toprule
base & data & update rule & steps & LR & MMLU; GSM8K \\
\midrule
Llama-3.1-8B & tulu-v2 shards & QLoRA SFT, $r{\in}\{8,64\}$ &
500/2000 & $10^{-4}$ & .62--.64; .50--.56 \\
Mistral-7B-v0.3 & tulu-v2 shards & QLoRA SFT, $r{\in}\{8,64\}$ &
500/2000 & $10^{-4}$ & .56--.59; .35--.41 \\
Qwen2.5-7B & ultrafeedback & QLoRA DPO, $r{=}8$, $\beta{=}.1$ &
500/1000 & $5{\times}10^{-5}$ & .719--.721; .84--.86 \\
Qwen2.5-1.5B & tulu-v2 shards & full FT / LoRA $r{=}64$ &
2000 & $2{\times}10^{-5}$/$10^{-4}$ & .59--.60; .14--.41 \\
\bottomrule
\end{tabular}
}
\end{center}

LoRA adapters use $\alpha = 2r$, dropout $0.05$, all-linear targets;
effective batch $16$ sequences, cosine schedule, warmup ratio $0.03$,
maximum length $1{,}024$ tokens. The DPO accuracy column is the
tightest in the table --- the update rule that moves accuracy least is
also the one that decorrelates least, the same phenomenon at a third
scale.

\Needspace{5\baselineskip}
\section[Worked example: auditing one claim with selective-evals]{Worked example: auditing one claim with \texttt{selective-evals}}
\label{app:toolexample}

One CLI invocation, with the claim's observed margin and measured
comparator correlation, returns the verbatim verdict sheet below; its
sensitivity boundaries are in Figure~\ref{fig:sensitivity}b.

Reading: the \texttt{exact} row is the plug-in verdict --- certified at
$k{=}27$ only under $\rhow \ge 0.69$ --- while \texttt{svbb} is the
selected-$\beta$ Berger--Boos $p$-value, already including the
$+\gamma$ spend, so it is compared to $\alpha = 0.05$: it clears only
at the top anchor, matching that adjustment's required floor of $0.756$
(the selected-beta threshold of Figure~\ref{fig:sensitivity}a). The joint adjustment
(floor $0.849$) needs the raw margin and the $\tau$ inputs, and ships
as \texttt{selective-evals joint-svbb}. The claim is \emph{not
margin-certified} at the two primary anchors $0.42/0.56$. The plug-in
budgets (13--17) exceed Bonferroni's (11), but are not joint-adjusted
budgets. The displayed $165$ is also a plug-in diagnostic: it does not
extend the $K\le100$ nuisance-adjusted reporting range.

\par\noindent\begin{minipage}{\linewidth}
{\small
\begin{verbatim}
$ selective-evals audit --z 2.61 --rho-b 0.567 --k 27 \
                        --rho-w 0.42 0.56 0.71 0.86 --se-rho-b 0.014
{
 "z": 2.61, "rho_b": 0.567, "k": 27, "alpha": 0.05,
 "naive": true, "independence": false,
 "bonferroni": false, "dunnett": false,
 "exact":      {"0.42": false, "0.56": false,
                "0.71": true,  "0.86": true},
 "svbb":       {"0.42": 0.0996, "0.56": 0.0796,
                "0.71": 0.0570, "0.86": 0.0345},
 "kbar_exact": {"0.42": 13, "0.56": 17,
                "0.71": 30, "0.86": 165},
 "kbar_bonferroni": 11
}
\end{verbatim}
}

\end{minipage}\par

The same audit runs from raw data via
\texttt{selective-evals audit -{}-winner-csv a.csv -{}-runner-csv b.csv
-{}-rho-w 0.56 0.71}, which computes $z$ and $\hat\rhob$ from two
aligned per-item correctness columns before rendering the verdict sheet
(the $\rhow$ grid is required --- floors are the caller's stated
assumptions, not tool defaults). Multi-benchmark input takes
\texttt{-{}-groups-csv} with per-item benchmark labels, which switches
$z$ and $\hat\rhob$ to the score-matched estimator of
\S\ref{sec:measurement} rather than naive pooling. No numerical entry
point substitutes a correlation outside the certificate domain: audit
commands return \emph{model-insufficient}, while mathematical primitives
raise \texttt{ModelDomainError}; values outside $[-1,1]$ are rejected as
malformed inputs.

\section{Arena, provenance, and metadata}
\label{app:audit-extra}

\textbf{Heterogeneity-adjusted verdicts.} Re-running the v1 board at
$\beta^+ = \hat\rhob + \Delta$ for $\Delta \in \{0.05, 0.10,
0.121\}$ (median, $2\times$-median, and maximum measured excess)
changes \emph{no} certification at the measured sensitivity anchors; the only
movement is $46/116/167$ of $788$ anchor-verdicts shifting from
\emph{refused} to \emph{model-insufficient}. The plug-in rank-1 budgets
tighten but hold ($11/13/18/66$ at $\Delta = 0.121$, vs.\
$13/17/30/165$ homogeneous), and both axes at their curated extremes
($\Delta = 0.121$, $R = 1.16$; the controlled families reach $R = 1.34$,
reported in Appendix~\ref{app:dial}, not imported into a board audit)
give M-I$/12/16/43$ --- the $0.42$ anchor exits the model's domain,
and wherever the model applies the budget never falls below
Bonferroni's $11$. The measured-extreme joint stress test ---
selected $\beta$, worst-cased standardization, and
$(\Delta, R) = (0.121, 1.16)$, using interval-certified
standardized-margin floors at the stated
$(\rhow^-, \beta^+, R)$ --- is the most conservative reported
scenario: the rank-1 required floor reaches $0.939$, every anchor
refuses ($p = 0.071$ even at the $0.86$ mean-band reference), and the
budgets fall to M-I$/6/8/13$, the inner ones equal to Bonferroni
evaluated at the same worst-cased margin ($\kbar = 6$). Because
$\Delta$ and $R$ bound nothing for a named hidden family, this is a
stress test at measured extremes, not an upper bound, and it is reported
alongside --- not instead of --- the primary adjustments. Wherever the model applies, the analysis is the audit protocol of \S\ref{sec:audit} unchanged.

\textbf{Arena (illustrative).} We fit Bradley--Terry to the public
135{,}632-battle release (52 models; ties as half-wins;
production pipeline not reproduced). Battle-bootstrap correlations of
same-provider score pairs are $-0.05$ to $-0.02$ over
five pairs ($B = 2{,}000$, SE $\approx 0.02$); the BT sum-to-zero
constraint alone induces $\approx -0.02$, so the constraint-adjusted
values are within two SEs of zero and $\rhow \gtrsim 0.07$ is
disfavored --- though the result remains partly structural (disjoint
battles). With those limits, Arena is consistent with $\rhow \approx 0$, where the exact correction approaches the
independence heuristic --- the case where each hidden variant requires
the largest correction. Adjacent top $z$-values are
$0.04$--$1.4$ except the rank-1 gap ($z = 2.47$; a
\emph{same-provider} preview/production duo) --- naive-certified,
refused at $k = 27$ near that corner --- and one mid-top gap
($z = 3.9$) that remains certified in every column. The experimental
Llama-4 variant of \citet{singh2025leaderboard}'s account sits
(uncertifiably, $z = 1.0$) inside the fitted top ten.

\textbf{Provenance and metadata.} ``The v2 board'' means the
4{,}576-model archived v2 snapshot frozen by
\citet{kim2025correlated}, whose \texttt{average\_score} column is the
official v2 average (the snapshot ordering is official by construction);
the v1 per-item matrix reproduces the official per-task metrics for all
top-20 models --- Kendall $\tau = 1.000$ against the recomputed
official six-task average; the retired page itself cannot be re-rendered
from web archives, which bounds any residual ordering discrepancy to the
display layer (Appendix~\ref{app:measurement}). No complete private
candidate logs are available in the audit inputs. Consequently, none of
the 394 v1 claims, or the v2/Arena illustrations, has verified
fixed-family applicability in this audit. Their numerical verdicts are
conditional scenarios, not endorsements of a verified development
protocol; deployment applicability remains \emph{metadata-insufficient}. We attempted to sharpen the default $k$ to per-base-model public
counts and found the metadata itself insufficient: submitter-declared
\texttt{base\_model} is non-empty in only 16.2\% of the v1
request history. Organization-level counts describe public submissions,
not hidden per-release search; the gap motivates the disclosure
reforms of \citet{singh2025leaderboard}.

\begin{figure}[!htbp]
\centering
\includegraphics[width=\linewidth]{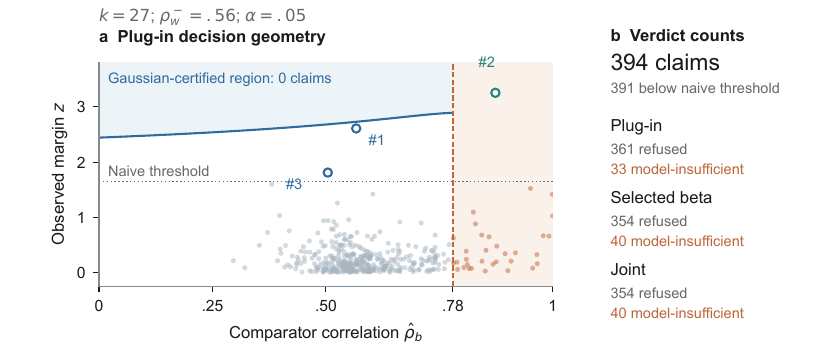}
\caption{\textbf{Most margins fail before selection is considered.}
All 394 overlapping v1 adjacent claims. Pale blue: Gaussian-certified
region; pale orange: outside the model domain. Claim \#2 (green) is
Bonferroni-supported despite model insufficiency. No claim passes
the shown Gaussian adjustments; (a) shows plug-in geometry only; (b)
reports each adjustment. Selected nuisance caps alter the applicability
check and change seven refusals into abstentions; the plotted orange
region represents the 33 plug-in model-insufficient claims.}
\label{fig:audit}
\end{figure}

\textbf{Sensitivity to the assumed multiplicity.} Recomputing every
verdict at $k \in \{5, 10, 15, 20, 27\}$, with the caps
$U_W(\gamma/k)$ and the joint $\gamma/2k$ quantiles recomputed per
$k$ (\texttt{k\_sensitivity}): the v1 board aggregate is essentially
$k$-invariant --- at the curated floor the (certified, refused,
model-insufficient) counts move only from $(1, 360, 33)$ at $k = 5$ to
$(0, 361, 33)$ at $k = 27$ on the plug-in adjustment, and analogously on the
others --- and the \emph{only} claim anywhere on the board that is
newly certified at any smaller $k$, any anchor, any adjustment, is the
rank-1 claim itself. Its profile at the curated floor: certified at
every adjustment for $k \le 5$, at the plug-in and selected-$\beta$ adjustments
for $k \le 10$, plug-in only at $k = 15$, refused everywhere from
$k = 20$. The audited v2 verdicts do not move at all ($\#1$
model-insufficient and $\#4$ certified at every $k$ in the grid). The
$k$-dependence is thus itself the diagnostic the curve reports: this
board's one contested claim is certifiable against a provider who
tried five variants and not against one who tried twenty.

\begin{figure}[!htbp]
\centering
\includegraphics[width=\linewidth]{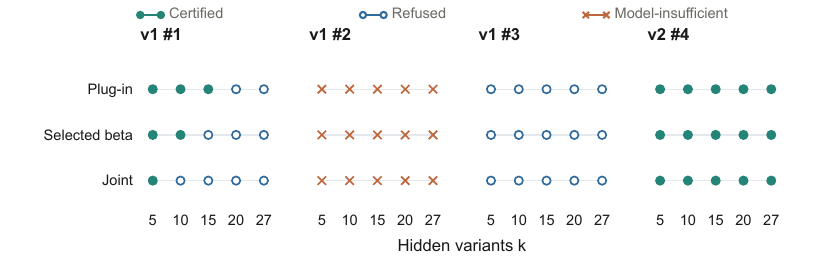}
\caption{\textbf{The same certificate yields four practical outcomes.}
Each panel uses $\rhow^-=0.56$, $\alpha=0.05$, the same five hidden
counts, and the same three nuisance treatments. Filled circles: certified; open circles: refused; crosses: model-insufficient. The nuisance caps are recomputed
at every $k$. Claim v1 \#2 remains supported by Bonferroni throughout
this grid even though these Gaussian columns are outside their
domain; v2 \#4 remains certified. The two boards use different
score rules and item sets, so these are case comparisons, not a
pooled success rate. Exact $k=27$ values are in Table~\ref{tab:audit}.}
\label{fig:audit-cases}
\end{figure}

\textbf{The v2 top-30 extension.} Extending the v2 audit to the
top-30 adjacent cross-provider claims (31 models' per-item details,
the same $21{,}606$ aligned items) certifies $5/30$ naively, $2/30$ by
Bonferroni, and $2/30$ at the curated floor at $k = 27$
(\texttt{e2\_v2\_results\_top30}). Where the densified board keeps a
top-15 pair unchanged ($6$ of $15$), $z$, $\hat\rhob$ and every
verdict are identical to the primary audit; the other nine claims
re-pair to nearer other-provider neighbours, shrinking their adjacent
margins --- the expected mechanics of a denser board, and the reason
the primary audit fixes its pairs before extending.

\textbf{Additional limitations.}
The 394 claims overlap, so board counts are descriptive.
The 4/12 pooled-vs-score-matched sign reversals
limit how far any single correlation ranking can generalize across
score functionals. Every $\rhow$ here is measured on public community
artifacts; we have no measurements from private frontier-lab checkpoint
families. The $(\beta^+, R)$ values are assumptions, not estimates: the
axes are unidentified for a hidden family --- public families
calibrate plausible values ($R$ reaches $1.34$ in the controlled families,
above the curated $1.16$) without bounding any of them; likewise the
plug-in guarantees hold for true margin SDs, covered formally only by
the joint selected-nuisance region and otherwise empirically
(Appendix~\ref{app:masking}).

\section{Extended related work}
\label{app:related}

\textbf{Leaderboard fragility and evaluation UQ.} A second attack
mode on the same platform, crowdsourced-vote manipulation
\citep{min2025rigging}, attacks the comparisons themselves rather than
the choice of which variant to submit, so no margin-based correction
can address it. \citet{messing2026hidden} attacks the
same problem from the variance side: standard intervals omit
variability from judge choice, decoding temperature and prompt design,
and a pipeline that fails to average over those sources leaves more
room for benchmark gaming --- quantified there as $56$ Elo at the
same documented $k = 27$, reducible to $32$ by pipeline design. That is
an \emph{ex ante} correction to the measurement pipeline; the certificate is the
\emph{ex post} correction to the claim. The two compound rather than
substitute: an interval shorter than the pipeline warrants makes hidden
selection require a larger correction, not a smaller one.

\textbf{Selective inference with full observability.} Locally
simultaneous inference restricts the simultaneity correction to the
questions the data could plausibly have raised \citep{zrnic2024locally};
the zoom correction \citep{zrnic2024zoom} inverts a ``zoom test'' that
spends error budget only on competitive candidates, recovering an
uncorrected interval when the winner's gap is large --- but requires
observing every candidate's score, the assumption our setting drops.
Conditional winner inference with general covariance is
\citet{bakshi2026flexible}; \citet{neufeld2026selective} review
methods for inference conditional on selection, unifying sample splitting, data
carving, and randomized CSI under a single information-splitting
framework. SIREN \citep{xu2026siren} corrects the
winner's curse prospectively by sample splitting; splitting is
impossible retrospectively, which is the setting of this audit. Where
candidates \emph{are} observed we import rank-verification machinery
directly \citep{hung2019rank,sood2025powerful}.

\textbf{Adaptivity and mechanism design.} The Ladder
\citep{blum2015ladder}, its randomized improvement
\citep{hardt2017climbing}, and the reusable holdout
\citep{dwork2015reusable} are ex-ante mechanism-design complements to
ex-post certification; \citet{roelofs2019meta} bounds the practical
stakes of sequential adaptive probing.
\citet{mccloskey2024critical} apply the same remedy in a different
literature --- raise the threshold until inference remains valid after
search the analyst never sees --- but put the unknown elsewhere: they model the searcher's
behavior and calibrate it, returning a threshold (in medical sciences,
roughly the classical value at $\alpha/5$), where we leave $k$ and
$\rhow^-$ as declared axes and return a curve. No behavioral
calibration recovers the correlation among searches, because it is a
property of the candidate family rather than of the searcher.

\end{document}